\documentclass[11pt]{article}

\usepackage[preprint]{acl}

\usepackage{hyperref}  
\usepackage{times}
\usepackage{latexsym}
\usepackage{amssymb}
\usepackage{amsthm}
\usepackage[T1]{fontenc}

\usepackage[utf8]{inputenc}

\usepackage{microtype}

\usepackage{inconsolata}

\usepackage{graphicx}
\usepackage{float}  %
\usepackage{xspace}
\usepackage{enumitem}
\usepackage{booktabs}
\usepackage{multirow}
\usepackage{pifont}
\usepackage{amsmath}
\usepackage{bm}
\usepackage{xcolor}
\usepackage{soul}  %
\usepackage{textcomp}
\usepackage{placeins}  %

\renewcommand{\dbltopfraction}{0.9}

\newcommand{\method}{{\sc eXTC}\xspace}
\newcommand{\spo}{{\sc SPO4SOP}\xspace}

\usepackage[most]{tcolorbox}
\usepackage{needspace}                       %
\definecolor{promptgreen}{HTML}{8EC4B0}      %
\definecolor{promptbg}{HTML}{FFFFFF}         %
\newtcolorbox{prompttemplatebox}[2][]{
  enhanced,
  colback=promptbg,
  colframe=promptgreen,
  colbacktitle=promptgreen,
  coltitle=white,
  fonttitle=\bfseries,
  title={#2},
  boxrule=0.9pt,
  arc=2pt,
  left=2mm, right=2mm, top=1.5mm, bottom=1.5mm,
  #1
}

\usepackage{algorithm}
\usepackage[noend]{algorithmic}

\newtheorem{problem}{Problem}
\newtheorem{definition}{Definition}
\newtheorem{lemma}{Lemma}

\newcommand{\cmark}{\ding{51}}%
\newcommand{\pmark}{(\ding{51})}%

\newcommand{\btheta}{\bm{\theta}}%

\newcommand{\cotho}{{\sc CoT}\xspace}

\newcommand{\mipro}{{\sc MIPROv2}\xspace}

\newcommand{\gepa}{{\sc GEPA}\xspace}

\makeatletter
\newcommand\footnoteref[1]{\protected@xdef\@thefnmark{\ref{#1}}\@footnotemark}
\makeatother

\newcommand{\algrule}[1][.5pt]{\par\vskip.5\baselineskip\hrule height #1\par\vskip.5\baselineskip}

\newcommand{\cbit}{\begin{compactitem}}
\newcommand{\ceit}{\end{compactitem}}
\newcommand{\cben}{\begin{compactenum}}
\newcommand{\ceen}{\end{compactenum}}

\newcommand{\beq}{\begin{equation}}
	\newcommand{\eeq}{\end{equation}}

\definecolor{darkgreen}{RGB}{41,166,41}

\newcommand{\bit}{\begin{itemize}}
	\newcommand{\eit}{\end{itemize}}
\newcommand{\ben}{\begin{enumerate}}
	\newcommand{\een}{\end{enumerate}}

\newcounter{x}

\newcommand{\bx}{\mathbf{x}}
\newcommand{\br}{\mathbf{r}}

\definecolor{celadon}{rgb}{0.67, 0.88, 0.69}
\definecolor{carolinablue}{rgb}{0.6, 0.73, 0.89}

\soulregister{\textbf}{1}  %
\newcommand{\hly}[1]{{\sethlcolor{red!10}\hl{#1}}}
\newcommand{\hlg}[1]{{\sethlcolor{cyan!14}\hl{#1}}}
\newcommand{\hlr}[1]{{\sethlcolor{red!25}\hl{#1}}}      %
\newcommand{\hlbest}[1]{{\sethlcolor{green!40}\hl{#1}}}   %
\newcommand{\hlrunup}[1]{{\sethlcolor{orange!20}\hl{#1}}}   %

\definecolor{aliceblue}{rgb}{0.867, 0.917, 0.964}
\definecolor{aliceyellow}{rgb}{0.999, 0.945, 0.796}
\definecolor{alicegray}{rgb}{0.844, 0.867, 0.898}

\newcommand{\fmodel}{f_{\text{model}}}
\newcommand{\fgrad}{f_{\text{grad}}}
\newcommand{\fup}{f_{\text{update}}}

\newcommand{\mFn}{\mathcal{F}^{-}} %
\newcommand{\mFc}{\mathcal{F}^{+}} %

\definecolor{dodgerblue}{rgb}{0.12,0.565,1}

\newcommand{\mX}{\mathcal{X}}
\newcommand{\mY}{\mathcal{Y}}
\newcommand{\mD}{\mathcal{D}}
\newcommand{\mR}{\mathcal{R}}
\newcommand{\mC}{\mathcal{C}}
\newcommand{\mG}{\mathcal{G}}
\newcommand{\mS}{\mathcal{S}}
\newcommand{\mP}{\mathcal{P}}

\newcommand{\Dbatch}{\mD_{\text{batch}}}

\newcommand{\nld}{{\Theta}_{\text{lang}}}

\newcommand{\grad}{\texttt{Gradient LLM}\xspace}
\newcommand{\up}{\texttt{Update LLM}\xspace}

\renewcommand{\hlbest}[1]{\colorbox{green!40}{#1}}
\renewcommand{\hlrunup}[1]{\colorbox{orange!20}{#1}}

\title{Structured Prompt Optimization Advances Globally and Locally Interpretable Learning from Complex Text 
}

\author{
  \textbf{Tianyang Zhou\textsuperscript{1}},
  \textbf{Wenbo Chen\textsuperscript{2,\dag}},
  \textbf{Pierre Jinghong Liang\textsuperscript{1}},
  \textbf{Leman Akoglu\textsuperscript{1}}
\\
  \textsuperscript{1}Carnegie Mellon University,
  \textsuperscript{2}Amazon
\\
  \small{\texttt{tzhou3@andrew.cmu.edu}, \texttt{wbchen@amazon.com}, \texttt{liangj@tepper.cmu.edu}, \texttt{lakoglu@andrew.cmu.edu}}
}

\begin{document}
\maketitle
\begingroup\renewcommand\thefootnote{\dag}\footnotetext{This work was conducted outside the author's role at Amazon.}\endgroup
\begin{abstract}

LLMs have advanced text classification, yet existing paradigms face a trade-off:
  supervised (label only) fine-tuning is scalable but offers limited reasoning on complex text and lacks broader model transparency, while discrete prompt optimization offers human-readable instructions but struggles with performance and scalability.
We introduce \method (\textsc{eX}plainable \textsc{T}ext \textsc{C}lassifier), centered on a Standard Operating Procedure (SOP, or rulebook) learned in natural language, in three stages:
(1) learning the SOP via a new Structured Prompt Optimization algorithm; (2) SOP-grounded reasoning distillation from a large teacher LLM into a compact LM; and (3) expanding reasoning capabilities beyond the SOP via reinforcement learning. 
This design enables \method to provide (i) fast inference via a compact LM, with (ii) inference-time local reasoning traces, alongside a global, modular explanation of its learned domain rules, while (iii) significantly outperforming existing paradigms across diverse benchmarks in both classification performance and explanation quality,
with stage-by-stage gains.

\end{abstract}

\section{Introduction}
\label{sec:intro}

Text classification is widely used in finance (e.g., market sentiment) \cite{DBLP:journals/mansci/DasC07}, healthcare (e.g., clinical note triage) \cite{jagannatha2016bidirectional}, law (e.g., case outcome prediction) \cite{chalkidis2020legal}, cybersecurity (e.g., spam/phishing/malware detection) \cite{salloum2022systematic}. With the rise of {LLM-generated} text, it is also critical for safety and reliability, e.g., detecting LLM-vs-human content \cite{DBLP:conf/icml/Mitchell0KMF23}, hallucination \cite{ji2023survey}, and harmful/unsafe generations \cite{gehman2020realtoxicityprompts}.

\looseness=-1 Despite its popularity, text classification on complex real-world tasks remains challenging due to (i) the high dimensionality of textual input, (ii) complex and nuanced semantics that extend beyond surface-level words, and (iii) long-range dependencies across text. Moreover, many applications require (iv) robust high-level discrimination rather than reliance on spurious low-level signals, as well as (v) interpretable explanations for individual decisions and global understanding of what the model has learned. %

\begin{figure*}[!t]
\vspace{-0.1in}
  \centering
  {\setlength{\fboxsep}{2pt}%
  \begin{minipage}[c]{0.62\linewidth}
    \fbox{\parbox[t]{\dimexpr\linewidth-4pt\relax}{\small\sloppy\emergencystretch=2em
      \textbf{(a) Example rule from the SOP (excerpt):} \hfill \textcolor{purple!70!black}{\textsf{[Global Explanation]}}\\
      \textbf{Trigger Pattern:} Apply \texttt{readmit} when a cancer patient is discharged while the disease is still \hly{unstable or under intensive treatment}. \ldots\\
      \textbf{Exceptions:} Patient \hlg{stable, with no cancer treatment started} and no new cancer-related complications. \ldots\\
      \textbf{Examples:} \ldots\vspace*{1pt}%
    }}%
  \end{minipage}\hfill
  \begin{minipage}[c]{0.36\linewidth}
    \centering
    \small
    \renewcommand{\arraystretch}{0.88}%
    \begin{tabular}{lrr}
      \toprule
      \textbf{(c) Method} & mF1 & Judge \\
      \midrule
      Zero-shot CoT       & 0.517 & 3.58 \\
      \cmidrule(lr){1-3}
      Stage I (SPO4SOP)   & 0.597 & 3.65 \\
      Stage I+II (R-SFT)  & \colorbox{orange!30}{\textbf{0.620}} & \colorbox{orange!30}{\textbf{4.29}} \\
      Stage I+II+III (\method) & \colorbox{green!40}{\textbf{0.652}} & \colorbox{green!40}{\textbf{4.31}} \\
      \bottomrule
    \end{tabular}
  \end{minipage}\\[1pt]
  \fcolorbox{gray!50}{gray!8}{\parbox[t]{\dimexpr\linewidth-5pt\relax}{\small\sloppy\emergencystretch=2em
    \textbf{(b) MIMIC Readmission Prediction.}\,\,
    \textit{Discharge summary:} Patient has liver spots suspicious for cancer. A biopsy was taken; patient was discharged stable \emph{while awaiting results}, with an \emph{outpatient} oncology clinic visit in 10 days.
  }}\\[1pt]
  \begin{minipage}[c]{0.46\linewidth}
    \fbox{\parbox[t][1.75cm][t]{\dimexpr\linewidth-4pt\relax}{\small\setlength{\fboxsep}{1pt}%
      \textbf{Zero-shot CoT reasoning:}\\
      ``Suspicious for metastatic disease with pending biopsy and planned oncology follow-up, \hly{likely needing chemotherapy/radiation soon}. High readmission risk.''\par\vfill\noindent
      \textit{Predicted:} \texttt{readmit}\,\,\ding{55}
    }}%
  \end{minipage}\hfill
  \textbf{\large v.s.}\hfill
  \begin{minipage}[c]{0.46\linewidth}
    \fbox{\parbox[t][1.75cm][t]{\dimexpr\linewidth-4pt\relax}{\small\setlength{\fboxsep}{1pt}%
      \textbf{\method (RL) reasoning:} \hfill \textcolor{purple!70!black}{\textsf{[Local Explanation]}}\par\vfill
      {``Biopsy done; results pending but no indication of active malignancy; \hlg{no intensive treatment in progress, patient stable} and ambulatory at discharge.''\par}\vfill\noindent
      \textit{Predicted:} \texttt{no\_readmit}\,\,\checkmark
    }}%
  \end{minipage}}%
  \vspace{-0.1in}
  \caption{\textbf{\method at a glance on MIMIC Readmission.}
  \textbf{(a)} a rule excerpt from our \textit{learned} rulebook (or SOP); %
  \textbf{(b)} on a test case: Zero-shot CoT over-fires on the surface trigger, whereas \method\ correctly invokes the rule's \emph{exception} (no treatment started yet); \textbf{(c)} both metrics improve monotonically across our three Stages. See Case Study in \S\ref{ssec:study}.}
  \label{fig:crownjewel}
  \vspace{-0.2in}
\end{figure*}

\looseness=-1 Traditional feature-based solutions \cite{joachims1998text} fall short on both fronts: fragmented features miss rich semantics, long-range context, and high-level reasoning, while their feature-importance and saliency explanations often miss the deeper rationale behind decisions \cite{lipton2018mythos}.

Large language models (LLMs) have transformed natural language processing (NLP), including text classification: BERT \cite{devlin2019bert}, one of the first Transformer-based LMs, already significantly improved classification across a wide range of NLP benchmarks, and many follow-up models build on this success \cite{vajjala2025text}.
LLMs offer ingrained world and domain knowledge, holistic contextual comprehension, a natural language interface, and high-level reasoning that is less reliant on low-level spurious signals and supports more interpretable predictions. %

\looseness=-1 However, existing solutions for LLM-driven classification remain limited in key aspects. In general, these approaches fall into two main categories: (1) supervised fine-tuning (SFT) of pretrained backbones, and (2) prompt optimization (PO), which learns human-readable commands in context that guide a frozen backbone \cite{DBLP:conf/emnlp/PryzantI0L0023,DBLP:conf/iclr/WangLW0LZJXH24} to generate the correct label tokens as outputs.
SFT is scalable, %
however, it does not readily provide reasoning traces or a global explanation of the underlying model. Chain-of-Thought prompting \cite{DBLP:conf/nips/Wei0SBIXCLZ22} of models fine-tuned only on labels does not readily produce competitive reasoning, especially for complex text. %
In contrast, hard (discrete) PO can yield human-readable prompts that serve as global explanations, but these are often unstructured. In addition, hard PO
is computationally demanding, making it less scalable for large-corpus tasks. Table~\ref{tab:paradigms} summarizes these trade-offs (see Table \ref{tab:related-work}, and Apdx. \ref{app:related} for detailed related work).

\begin{table}[H]
\vspace{-0.05in}
\centering
\caption{\method vs.\ common paradigms.}
\label{tab:paradigms}
\vspace{-0.1in}
\setlength{\tabcolsep}{3pt}
\small
\begin{tabular*}{\linewidth}{@{}l@{\extracolsep{\fill}}cccc@{}}
\toprule
Paradigm & \begin{tabular}[b]{@{}c@{}}Global\\Expl.\end{tabular} & \begin{tabular}[b]{@{}c@{}}Local\\Expl.\end{tabular} & \begin{tabular}[b]{@{}c@{}}Inference\\Efficiency\end{tabular} & Perf. \\
\midrule
CoT prompting & No & Low & Low & Low \\
Prompt Opt. & Unstructured & High & Low & Medium \\
Direct SFT & No & Low & High & High \\
\textbf{\method} & \textbf{Structured} & \textbf{High} & \textbf{High} & \textbf{High} \\
\bottomrule
\end{tabular*}
\vspace{-0.1in}
\end{table}

This work bridges these gaps by introducing \method, an effective and efficient LLM-powered
\textsc{eX}plainable \textsc{T}ext \textsc{C}lassifier
with human-readable global and local explanations.
At the core of \method is a structured rulebook, a decision set learned from labeled data that serves as the global explanation and grounds the rest of the pipeline (Fig.~\ref{fig:frame}): \method distills rulebook-guided reasoning into a small LM classifier (local explanation) and further improves it beyond the rulebook through reinforcement learning. Our main contributions are:

\vspace{-0.03in}
\begin{itemize}[noitemsep, topsep=0pt, leftmargin=1em]
    \item \textbf{Reasoning-centric Text Classification:} We introduce \method, a new LLM-powered approach to \textit{classification and reasoning over complex text} that provides both global and local explanations, which no existing paradigm offers (Table~\ref{tab:paradigms}).
(See example in Fig. \ref{fig:crownjewel} (a) and (b).)
    \item \looseness=-1 \textbf{Formulation and Algorithm Innovations:} \method has a multi-stage framework: (i) \textbf{SOP Learning} learns a rulebook via a new structured prompt optimization algorithm whose rule search scales linearly with the number of candidate rules; the learned SOP then grounds (ii) \textbf{SOP-Grounded Distillation}, which transfers rulebook-guided reasoning traces from a large LM onto a small LM, which (iii) further improves through reinforcement learning \textbf{Beyond SOP}, with a new GRPO variant, \textsc{BD-GRPO}, that targets the hard (i.e., teacher-failed) cases. (See Fig. \ref{fig:crownjewel} (c).)

    \item \textbf{Desirable Properties:} \method is (i) \textit{effective}, capitalizing on \spo, SFT and RL; (ii) \textit{efficient}, with a compact 4B-parameter student model obtained via distillation (training and inference cost in Apdx.~\ref{app:cost}); and (iii) \textit{doubly-interpretable}; rulebook constitutes a global explanation, and small LM's reasoning at inference serves as a local explanation per input.

    \item \textbf{Evaluation:} Experiments on three real-world benchmarks with ground-truth evidence for labels show that \method outperforms baselines in both classification performance and explanation quality, improving macro-F1 over CoT prompting by 10.8, 3.6 and 13.5 points, all statistically significant (Apdx.~\ref{app:significance}); ablations show that the learned SOP improves macro-F1 and explanation quality in Stages~II and III (\S\ref{ssec:ablation}).
\end{itemize}
\vspace{-0.025in}

\vspace{-0.025in}
\section{Problem and Preliminaries}
\label{sec:prelim}

\begin{figure*}[!t]
 \vspace{-0.1in}
  \centering
  \includegraphics[width=0.92\linewidth]{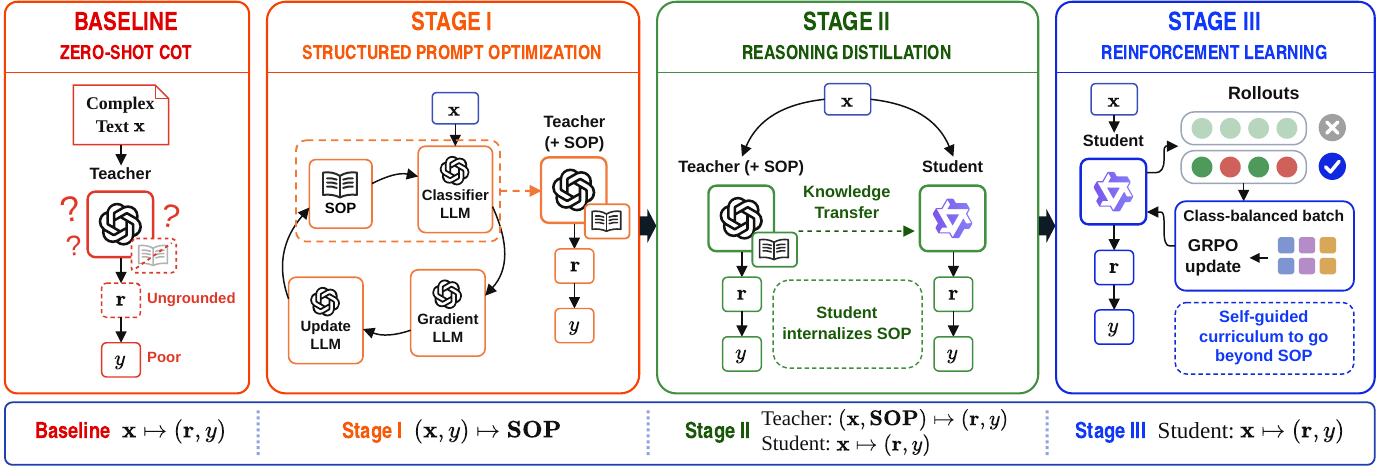}
  \vspace{-0.1in}
  \caption{\method Overview: \textbf{Stage I}: \spo learns a rulebook that \textit{grounds} \textbf{Stage II}: teacher-reasoning and distillation; \textbf{Stage III}: \textsc{BD-GRPO}, RL on class-balanced batches of informative rollouts, goes beyond the SOP. SOP serves as the global explanation, while
   the student LM classifies with local explanations (instance-level reasoning).}
  \label{fig:frame}
  \vspace{-0.26in}
\end{figure*}

\vspace{-0.05in}
\subsection{Problem Statement}
\vspace{-0.015in}

\noindent\textbf{Typical Text Classification:}
Let $\mX$ denote the input space with domain $\nld$ representing natural language, where $\bx_i \in \mX$ is a text document (a sequence of tokens).
Let $\mY$, $|\mY|=C$, denote the set of classes, where $y_i \in \mY$ is the class label of $\bx_i$.

Typical text classification aims to learn a model $\fmodel(\cdot;\btheta)$, %
given training set $\mD = \{(\bx_i,y_i)\}_{i=1}^n$, that achieves low classification error.

\noindent\textbf{Decision Sets: }
A decision set $\mR$ is an unordered
set of rules $\{R_1,\ldots,R_K\}$ that are \textit{stand-alone}, i.e., not linked via ‘else’ statements.
This makes them easier to interpret than hierarchical decision lists where a rule deep in a decision list is in fact a long rule and applies only when none of the preceding rules apply.
\looseness=-1 Unlike such a deep rule, a stand-alone rule fully specifies in its own text when it fires, so each rule can be read, audited, and edited on its own, without re-validating the rest.

\vspace{-0.05in}
\begin{definition}[\textbf{Decision Composition}]
\label{def:composition}
Each rule $R_j \in \mR$ has a target label $l_j \in \mY$ and either ``\emph{fires}'' for an input $\bx$ (predicting $l_j$) or ``\emph{abstains}''. Given a priority order $p: \mY $$\to$$ \{1,\ldots,|\mY|\}$ with default $y_{\text{def}} := \arg\min_{l} p(l)$, a decision set's prediction is

\vspace{-0.23in}
\begin{equation}
\fmodel(\bx; \mR) \;:=\; R_{\arg\max_j\, p(R_j(\bx))}(\bx)
\end{equation}
\vspace{-0.20in}

\noindent
where $R_j(\bx) := y_{\text{def}}$ when $R_j$ abstains, so if no rule fires, $\fmodel(\bx; \mR) = y_{\text{def}}$. In the binary case $(\mY=\{0,1\})$, this reduces to the standard max-pool: $\fmodel(\bx; \mR) \neq y_{\text{def}}$ iff at least one rule fires.
\end{definition}

Decision sets have been
effectively used for explainable ML for tabular data \cite{conf/kdd/LakkarajuBL16,macha2018explaining}, where each rule corresponds to an itemset; i.e., a conjunction of predicates over numerical and/or categorical features (e.g., \texttt{age}$>65$ $\wedge$ \texttt{location}$=$California).
Different from prior literature, our work introduces decision sets of rules in \textit{natural language}. This allows us to utilize them (1) as a structured prompt of an LLM-powered text classifier, and (2) as human-readable explanations behind decisions.

\vspace{-0.05in}
\begin{problem}[\textbf{Explainable Text Classification}]
    Given a labeled text corpus $\mD = \{(\bx_i,y_i)\}_{i=1}^n$, $\bx_i \in \nld$, estimate a text classification model %
    $\fmodel(\cdot;\btheta): \bx \mapsto \{\br, y\}$ that produces {\textbf{\textit{(1) instance-level local explanation}}} $\br \in \nld$ (in the form of reasoning), and {\textbf{\textit{(2) model-level global explanation}}} $\mR$$=$$\{R_k \in \nld\}$ (in the form of a decision set of natural language rules).
\end{problem}
\vspace{-0.05in}

Our problem statement differs from traditional text classification by explicitly requiring both local and global explanations in natural language, motivating a language-model-driven approach. Further, the global explanation is structured and decomposable, enabling readability and modular inspection. Such explanations provide a more coherent and richer reasoning than fragmented feature importance scores or saliency maps, which often only identify portions of the input text most associated with the predicted label, rather than offering
a deeper semantic interpretation of the content.

\vspace{-0.05in}
\subsection{Verbalized Machine Learning}
\vspace{-0.05in}

Traditional neural models for text classification exhibit numerical parameters $\btheta$ in continuous space.
Motivated by modern LLMs' prowess in many NLP tasks, recent ``verbalized'' machine learning approaches \cite{xiao2024verbalized,DBLP:journals/nature/YuksekgonulBBLLHGZ25}, employ a pretrained (frozen) LLM
 as a function approximator that is parameterized by its (text) prompt. In essence, both the input data $\mD $$\in$$ \nld$ and the model parameters $\btheta $$\in$$ \nld$ (i.e., the prompt) are represented in natural language.
 LLMs are further employed for prompt optimization, to derive ``textual gradients'' that guide learning the prompt through reasoning and reflection on model mispredictions across training epochs.

\vspace{-0.075in}
\section{Proposed Framework}
\label{sec:method}

\vspace{-0.05in}
\subsection{Motivation and Overview}
\vspace{-0.05in}

Fig. \ref{fig:frame} illustrates the three stages of \method: \textbf{I. SOP Learning} employs structured prompt optimization to learn a rulebook or SOP (Standard Operating Procedure); \textbf{II. SOP-Grounded Distillation} collects grounded reasoning traces and performs distillation; and \textbf{III. Beyond SOP} further improves reasoning beyond the learned rulebook through reinforcement learning.
{Stages I and II access the large LM only as a black box; the small LM is trained in Stages II and III.}
We provide an overview motivating each stage's role in the pipeline.

\noindent\textbf{Stage I. SOP Learning (\S\ref{ssec:spo})}
\looseness=-1 \textit{Motivation:} Given $\langle$text, label$\rangle$ training pairs $\mD $$=$$ \{(\bx_i,y_i)\}_{i=1}^n$, SFT of an existing language model (LM) maps input text to its label but does not readily produce individual explanations $\br_i$'s, while direct CoT prompting \cite{DBLP:conf/nips/Wei0SBIXCLZ22} yields poor classification and explanation quality on complex text (\S\ref{ssec:main_results}).

\textit{Overview:} As with RAG-based generations for knowledge-intensive tasks \cite{DBLP:conf/nips/LewisPPPKGKLYR020}, additional in-context information, such as a task-specific SOP (Standard Operating Procedure) or a rulebook could help obtain enriched reasoning.
To this end, we learn a rulebook: a structured prompt of
 stand-alone natural language rules (i.e., a decision set). We propose a new \textit{structured prompt optimization} algorithm based on verbalized machine learning (VML) \cite{xiao2024verbalized} to learn such a task-specific rulebook.

\noindent\textbf{Stage II. SOP-Grounded Distillation (\S\ref{ssec:rsft})}
\textit{Motivation and Overview:} Given the learned SOP, a teacher LLM is prompted with the enriched and grounded context $\langle \text{SOP}, \bx \rangle$ to generate reasoning and label pairs $\langle \br, y \rangle$. Retaining only teacher traces with correct predicted labels, we distill triplets ${(\bx_i, \br_i, y_i)} \subset \mD$ onto a smaller (student) LM via reasoning-supervised fine-tuning (R-SFT), training the model to map input text to both reasoning and label outputs. This distillation serves two purposes: (i) equipping the student model with \textit{SOP-free} reasoning capability, and (ii) producing a smaller and more efficient inference model. The former enables the model to move beyond the SOP through subsequent reinforcement learning (RL), while the latter facilitates efficient RL rollouts.

\noindent\textbf{Stage III. Beyond SOP (\S\ref{ssec:rl})}
\textit{Motivation:}
\spo's iterative optimization requires many costly LLM calls, so the rulebook is learned from a relatively small subset of training data, limiting the teacher model's reasoning capability and coverage. In turn, keeping only teacher traces with correct predictions limits distillation to a training subset.

\textit{Overview:}
To further improve reasoning capability, we employ RL, enabling the student-LM to generate multiple traces per input within each batch. The training set naturally decomposes into easy and hard examples, reflecting the teacher model's varying accuracy, and our goal is to rapidly improve performance on the harder cases without imposing an explicit manual curriculum.
To this end, we over-sample candidate inputs and filter out both overly easy and overly hard examples whose rollouts yield zero GRPO (Group Relative Policy Optimization) advantage \cite{DBLP:journals/corr/abs-2402-03300}. Constructing batches from examples with non-zero advantage promotes more informative learning signals, thereby improving RL's sample-efficiency and effectiveness.

\vspace{-0.05in}
\subsection{I: SOP Learning}
\label{ssec:spo}
\vspace{-0.05in}

The first goal is to construct a task-specific rulebook to be placed in-context with each test sample to \textit{ground} a language model's reasoning and classification.
The problem can be viewed as prompt optimization (PO), specifically as \textit{structured} PO, where the rulebook is to comprise \textit{standalone} decision rules expressed in natural language.
Since the rulebook governs the LM output, we also refer to it as a Standard Operating Procedure (SOP).

Algo. \ref{algo:spo4sop} presents \spo, our proposed structured prompt optimization approach to learning a task-specific rulebook (See diagram in Apdx. \ref{app:flow} Fig. \ref{fig:flow}).
In essence, it employs verbalized ML and text gradients to update the SOP by orchestrating a Classifier LLM, Gradient LLM, and Update LLM.

\vspace{-0.015in}
\noindent\textbf{\textit{Algorithm at a glance.}} \spo\ starts \emph{cold} with an empty decision set $\mR_0 = \emptyset$ and iterates over $T$ batches drawn from $\mD_{\text{train}}$. Each iteration handles two disjoint error modes. (i) \textbf{False coverage} (a rule fires on $\bx_i$ but its label $l_j$ disagrees with ground truth $y_i$): the Gradient LLM proposes \emph{exceptions} from the ``offending'' example, and the Update LLM rewrites the rule into a more precise candidate. (Lines~\ref{lin:fc-start}--\ref{lin:fup-rev}) (ii) \textbf{Blind spots} (no rule fires and the default label $y_{\text{def}}$ is wrong): the Gradient LLM extracts an \emph{error pattern} per missed sample, and the Update LLM synthesizes new candidate rules grouped by target label $l \in \mY$. Candidates from each iteration accumulate in a \emph{persistent pool} $\mP_t$ throughout the iterations, such that a rule generated early but not initially selected remains available for later subset search. (Lines~\ref{lin:bs-start}--\ref{lin:expand})

\vspace{-0.015in}
\noindent\textbf{\textit{Subset selection.}} From $\mP_t$ we select a small active set $\mR_t \subseteq \mP_t$ ($|\mR_t| $$\le$$ K$) that maximizes macro-$F_1$ on $\mD_{\text{val}}$ with an AIC-style sparsity penalty $\lambda \cdot |\mR_t| / |\mD_{\text{val}}|$. Because each rule's per-sample prediction is cached (Line~\ref{lin:classify}), evaluating any candidate subset reduces to composing cached predictions, thanks to the composability of decisions across rules (Definition~\ref{def:composition}). As such, the LLM cost is \textit{linear} in the pool size $|\mP_t|$ rather than exponential in subset count. We approximate the argmax with beam search (Line~\ref{lin:subselect}); with details in Apdx.\ \ref{app:subselect}. After $T$ iterations, \spo\ returns $\mR_T$ as the learned rulebook.
Importantly, it serves as a global explanation of what the model has learned.

\renewcommand\algorithmicrequire{\textbf{Given:}}

\renewcommand\algorithmicensure{\textbf{Output:}}

\begin{algorithm}[!t]
	\caption{\spo:  \textbf{Structured Prompt Optimization for Standard Operating Procedure}\label{algo:spo4sop}}
	\small{
	\begin{algorithmic}[1]
		\REQUIRE  $\mD_{\text{train}} = \{(\bx_i,y_i)\}_{i=1}^n$, $\mD_{\text{val}}$, max iter.s $T$, max subset size $K$, sparsity penalty $\lambda$, priority $p: \mY \to \{1,\ldots,|\mY|\}$ (default  $y_{\text{def}} := \arg\min_{l} p(l)$)
		\ENSURE SOP (i.e. decision set $\mR$ in natural language)
		\item[\textbf{LLMs:}] $\fmodel$ Classifier, $\fgrad$ Gradient, $\fup$ Update (prompts in Apdx.~\ref{ssec:fmodel})
		\algrule
		
\STATE Init. SOP, pool,  candidates $\mR_0 $$=$$ \mP_0 $$=$$ \mC $$:=$$ \emptyset$  \COMMENT{cold start}

 \FOR{$t=1,\ldots,T$, batch $\mD_{\text{batch}} \subseteq \mD_{\text{train}}$}
        
\FOR{$R_j\in \mR_{t-1}$   } \label{lin:fc-start}
\STATE $\mFc_j := \{\bx_i\in \mD_{\text{batch}} \;|\; \fmodel(\bx_i; R_j) = l_j \neq  y_i  \}$

  \ENDFOR  

\textcolor{dodgerblue}{$\blacktriangleright$ \textit{\textbf{False coverage}: add exceptions to mis-firing rules}}
\FOR{$R_j\in \mR_{t-1}$} 
\STATE  $\mG_j := \emptyset$
\FOR{$\bx_i \in \mFc_j$} 
\STATE $\mG_j := \mG_j\; \cup$ $\fgrad(\bx_i, y_i; l_j, R_j)$ 
\ENDFOR
\STATE  $\;\mC :=\mC \cup \fup(R_j, \mG_j)$ \label{lin:fup-rev}

\ENDFOR  
	
\textcolor{dodgerblue}{$\blacktriangleright$ \textit{\textbf{Blind spots}: add new rules}}
\STATE $\mFn := \{\bx_i \in \mD_{\text{batch}} \;|\; \fmodel(\bx_i; \mR_{t-1}) $$=$$ y_{\text{def}} \neq y_i\}$ \label{lin:bs-start}

\FOR{$\bx_i \in \mFn$}
\STATE $g_i :=\fgrad(\bx_i, y_i; \{R\in \mR_{t-1} \;|\; l = y_i\})$  

\ENDFOR
\FOR{ $l \in \mY$ s.t. $\mG_l := \{g_i : \bx_i \in \mFn,\, y_i = l\} \neq \emptyset$}
\STATE $\mC := \mC \cup \fup(\mG_l;\, l)$ \label{lin:fup-new}

\ENDFOR
\STATE	Expand persistent pool; $\mP_t := \mP_{t-1} \cup \mC$ \label{lin:expand}

			\algrule
		\textcolor{dodgerblue}{$\blacktriangleright$ \textit{\textbf{Rule subset selection} on $\mD_{\text{val}}$}}
            \STATE $\widehat{y}_{ij} := \fmodel(\bx_i; R_j)$ for $\bx_i \in \mD_{\text{val}}$, $R_j \in \mC$ \label{lin:classify}
  \STATE	$\mR_{t} $$:=$$ \mathop{\mathrm{argmax}}\nolimits_{\mathcal{S} \subseteq \mP_t, |\mathcal{S}| \leq K} [\mathcal{F}(\mathcal{S};\mD_{\text{val}}) $$-$$ \lambda |\mathcal{S}| / |\mD_{\text{val}}|]$ \label{lin:subselect}
  \COMMENT{beam search, Apdx.~\ref{app:subselect}; $\mathcal{F}$: macro-$F_1$}

   \ENDFOR
		\STATE \textbf{Return} $\mR_T$ \label{lin:return}
	\end{algorithmic}
	}
\end{algorithm}
 \setlength{\textfloatsep}{0.05in}

\vspace{-0.015in}
\noindent\textbf{\textit{Efficiency.}~}
\spo's key compute cost is the number of required LLM API calls, which may directly translate into monetary cost in practical deployments. Therefore, we analyze the time complexity of \method in terms of the \#LLM-calls.

\vspace{-0.05in}
\begin{lemma}[Scalable subset search]
\label{lem:subset}
\spo's rule subset search scales linearly with the number of candidate rules $m$, rather than with the exponential number of possible rule subsets. This is enabled by decision composition (Definition~\ref{def:composition}), which requires only a linear number of LLM calls in $m$ to obtain decisions for arbitrary rule subsets.
\end{lemma}
\vspace{-0.175in}

\begin{proof}
\textit{See Apdx.\ \ref{app:efficiency}}.
\end{proof}
\vspace{-0.1in}

As a qualitative study, Fig.~\ref{fig:tsne} shows 2D embedding of PO-classifier reasonings on MIMIC. %
\spo's reasonings form well-separated clusters corresponding to individual rules, while those from baseline PO methods are diffused and unstructured, showing that \spo's standalone rules induce distinct reasoning, so each decision can be traced to a specific rule of the global explanation.

\vspace{-0.1in}
\subsection{II: SOP-Grounded Distillation}
\label{ssec:rsft}
\vspace{-0.05in}

\noindent\textbf{\textit{Teacher trace generation and rejection sampling.}}
To absorb SOP's task knowledge into a compact student model, we collect SOP-grounded reasoning traces from a teacher LLM. A natural choice is to use the same LLM as the Stage~I Classifier, ensuring consistent SOP interpretation across stages. %
\begin{figure}[t]
 
  \hspace{-0.15in} \includegraphics[width=1.05\linewidth]{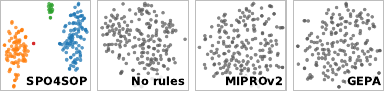}
  \vspace{-0.25in}
  \caption{t-SNE of PO-classifier reasoning on MIMIC: %
  \spo's embeddings form distinct clusters per rule, while those of baselines are diffused and unstructured.}
  \label{fig:tsne}
\end{figure}We ground the context by concatenating the learned SOP with each training example as $\langle \text{SOP}, \bx_i \rangle$, and prompt the teacher for a reasoning-label pair. We draw at most $M$$=$$4$ candidate traces per example and accept the first whose predicted label matches ground truth as the distillation trace $(\br_i, y_i)$ for $\bx_i$. Examples with all $M$ mispredictions are denoted as \emph{hard} and excluded from R-SFT, while the other \emph{easy} ones form the distillation set. This hard/easy partition further seeds the RL curriculum in \S\ref{ssec:rl}. %

\noindent\textbf{\textit{Reasoning-supervised fine-tuning (R-SFT).}}
Letting $\mD_{\text{rsft}} = \{(\bx_i, \br_i, y_i)\}$ denote the accepted (easy) triplets, we fine-tune a smaller student LM with LoRA \cite{hu2022lora} via next-token cross-entropy on $[\br_i; y_i]$ conditioned on $\bx_i$. To give the subsequent RL stage a stable starting point, we apply class-balanced upsampling on $\mD_{\text{rsft}}$: each epoch traverses the majority class once while minority classes are oversampled to match. Training details are in Apdx.~\ref{app:config}.

After R-SFT, the student maps input text directly to both reasoning and label outputs without needing the SOP in-context at inference: i.e., SOP's task knowledge has been internalized into the model's parameters. Student's per-example reasoning then serves as a \textit{local explanation} of its prediction, complementing the SOP's \textit{global explanation} (\S\ref{ssec:spo}).

\vspace{-0.075in}
\subsection{III: Beyond SOP via RL}
\label{ssec:rl}
\vspace{-0.05in}

\noindent\textbf{\textit{Why standard GRPO underperforms on classification.}}
Stage~II equips the student with SOP-grounded reasoning but leaves the \emph{hard} examples (\S\ref{ssec:rsft}) unaddressed. We turn to RL to extend coverage by letting the student self-explore. While GRPO-style RL has been most actively studied on verifiable reasoning tasks such as math and code, where rollouts span correct and incorrect outcomes \cite{DBLP:journals/corr/abs-2402-03300,DBLP:journals/corr/abs-2503-14476}, text classification with a small and often imbalanced label space exposes a distinct failure mode: \textit{rollout homogeneity}. \begin{table}[t]
\centering
\caption{Distribution (\%) of rollout outcomes within each hard/easy group at the start of RL ($G$=8), shown for ICLR Review as a representative example; the other two datasets follow the same pattern (see Apdx.~\ref{app:rldyn}, Tab.~\ref{tab:hardness_bucket_full}). Only the \colorbox{yellow!30}{$1$\textendash $7/G$} column carries non-zero GRPO advantages and provides gradient signal.}
\vspace{-0.1in}
\label{tab:hardness_bucket}
\setlength{\tabcolsep}{4pt}
\setlength{\fboxsep}{2pt}%
\small
\begin{tabular}{llrrr|r}
\toprule
 &  & \multicolumn{3}{c}{Rollouts correct (\%)} & \\
\cmidrule(lr){3-5}
Dataset & Diff. & 0/G & \colorbox{yellow!30}{$1$\textendash $7/G$} & G/G & \% of band \\
\midrule
ICLR Review & hard & 54.0 & \colorbox{yellow!30}{43.6} & 2.4 & \colorbox{yellow!30}{29.5} \\
 & easy & 0.7 & \colorbox{yellow!30}{18.1} & 81.2 & \colorbox{yellow!30}{70.5} \\
 & all & 8.6 & \colorbox{yellow!30}{21.9} & 69.5 & (14.8\% nat.) \\
\bottomrule
\end{tabular}
\end{table}
\relax Within a group of $G$ rollouts on the same input, all rollouts often share the same predicted label, yielding a zero-advantage group, hence no gradient signal. Concretely, each rollout $i$ receives a binary correctness reward $R_i = 2\cdot\mathbf{1}[\widehat{y}_i = y] - 1 \in \{-1, +1\}$, from which GRPO computes a within-group normalized advantage
\begingroup
\setlength{\abovedisplayskip}{1pt}\setlength{\belowdisplayskip}{1pt}
\begin{equation}\small
\hat{A}_i \;=\; \frac{R_i - \mathrm{mean}(\{R_j\}_{j=1}^{G})}{\mathrm{std}(\{R_j\}_{j=1}^{G}) + \varepsilon},
\label{eq:adv}
\end{equation}
\endgroup
and optimizes the clipped surrogate
\begingroup
\setlength{\abovedisplayskip}{1pt}\setlength{\belowdisplayskip}{1pt}
\setlength{\jot}{-2pt}
\begin{multline}\small
\mathcal{J}_{\text{GRPO}}(\theta) = \mathbb{E}\Big[\tfrac{1}{G}\!\sum\nolimits_i \min\!\big(\rho_i \hat{A}_i,\\
\text{clip}(\rho_i, 1{-}\epsilon, 1{+}\epsilon)\hat{A}_i\big)\Big] - \beta\, D_{\text{KL}}[\pi_\theta \| \pi_{\text{ref}}]
\label{eq:grpo}
\end{multline}
\endgroup
with importance ratio $\rho_i = \pi_\theta(\br_i|\bx) / \pi_{\theta_{\text{old}}}(\br_i|\bx)$, and $\varepsilon$ a small constant. When all rollouts in a group share the same prediction, $\mathrm{std}(\{R_j\}) = 0$ and every $\hat{A}_i$ vanishes. To quantify this at the start of RL, we run $G{=}8$ rollouts per training example with the Stage~II R-SFT model and bucket each example by its number of correct rollouts. Taking ICLR Review as a representative example (Table~\ref{tab:hardness_bucket}; analogous patterns hold on the other two datasets, Apdx.~\ref{app:rldyn}), easy examples are mostly already mastered and a majority of hard examples are unsolved, leaving only a thin informative band of $21.9\%$ where rollouts within a group disagree; implying that most of a \textit{random} GRPO batch contributes no gradient.

\begin{table*}[!t]
    \centering
    \vspace{-0.1in}
   \setlength{\tabcolsep}{3pt} \caption{Dataset Overview and Summary Statistics}
    \vspace{-0.1in}
\label{tab:dataset}
  \scalebox{0.72}{
    \begin{tabular*}{1.389\textwidth}{@{\extracolsep{\fill}}l c c c c l l l}
        \toprule
        \textbf{Dataset} & \textbf{Split (Tr/V/Te)} & \textbf{Avg Tok.} & \textbf{Min Class (\%)} & \textbf{Class} & \textbf{Domain} & \textbf{Task Description} & \textbf{Evidence Type} \\
        \midrule
       ContractNLI & 7191/1037/2091 & 2185 & 11 & 3 & Legal & Contract  NLI & Contract Spans \\ %
        \midrule
      ICLR Review  & 8552/1000/655 & 2242 & 37 & 2 & Peer Review &  Paper Acceptance  & Meta-Review \\ %
        \midrule
      MIMIC Readmission  & 8000/1000/1000 & 3209 & 22 & 2 & Clinical & Patient Readmission  & LACE Risk Factors \\ %
        \bottomrule
    \end{tabular*}
    }
    \vspace{-0.2in}
\end{table*}

\noindent\textbf{\textit{Balanced Dynamic GRPO (\textsc{BD-GRPO}).}}
We introduce \textsc{BD-GRPO}, a classification-aware variant of GRPO that resolves rollout homogeneity through two coupled mechanisms operating in a specific order. First, training pool is partitioned by class; for each class $c$ with per-class quota $n_c$, we draw a candidate batch $\widetilde{\mathcal{B}}_c$ of size $\kappa n_c$ from $\mathcal{D}_{\text{train}}^{(c)}$ ($\kappa{>}1$ the oversample factor; value in Apdx.~\ref{app:config}) and run $G$ rollouts per candidate. Second, in the same spirit as DAPO's dynamic sampling \cite{DBLP:journals/corr/abs-2503-14476}, we discard groups with $\mathrm{std}(\{R_j\}) = 0$ and take the first $n_c$ survivors per class:
\begingroup
\setlength{\abovedisplayskip}{1pt}\setlength{\belowdisplayskip}{1pt}
\begin{equation}\small
\mathcal{B} = \bigcup_{c \in \mathcal{Y}}\, \text{first}_{n_c}\!\left\{\bx \in \widetilde{\mathcal{B}}_c \,:\, \mathrm{std}(\{R_j(\bx)\}_{j=1}^{G}) > 0\right\},
\label{eq:bdgrpo}
\end{equation}
\endgroup
producing a class-balanced batch of only informative groups for the GRPO update (Eq.~\ref{eq:grpo}). \begin{figure}[t]
  \vspace{-0.05in}\centering\includegraphics[width=\linewidth]{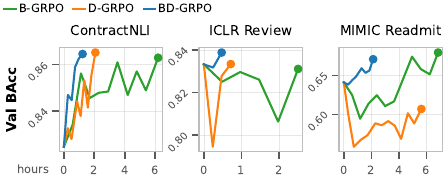}
   \vspace{-0.3in}
  \caption{Bal-acc vs. time, truncated at each run's peak. \textsc{BD-GRPO} peaks $3$-$5\times$ faster than \textsc{B-GRPO} (w/o dyn.\ filter); \textsc{D-GRPO} (w/o balancing) collapses on MIMIC.}
  \label{fig:rl_wallclock}
\vspace{-0.075in}
\end{figure}
\begin{figure}[t]
  \centering
   \vspace{-0.04in}
  \includegraphics[width=\linewidth]{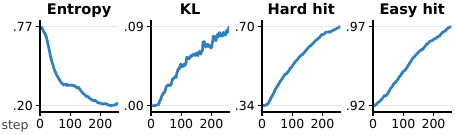}
 \vspace{-0.27in}
  \caption{\textsc{BD-GRPO} training on ContractNLI (other datasets in Apdx.~\ref{app:rldyn}): entropy stays bounded and KL small, while the \textit{hard hit rate steadily increases} and the easy hit rate remains high.
  }
  \label{fig:rl_main}
\end{figure}
\relax If a class yields fewer than $n_c$ informative survivors, we top up its quota with random draws from the same class. Performing dynamic sampling \emph{after} per-class oversampling is critical under severe label imbalance: replacing it with class-agnostic oversampling (\textsc{D-GRPO}) causes RL to over-predict the majority class and collapse on MIMIC (Fig.~\ref{fig:rl_wallclock}).

\noindent\textbf{\textit{Self-guided curriculum and wall-clock time.}}
Together, the two mechanisms implement an \emph{implicit self-guided curriculum}: at every step, the model trains on examples in its current informative band rather than on random batches. This band is itself hard-heavy: as Table~\ref{tab:hardness_bucket} shows for ICLR Review, hard examples take a far larger share than in the natural training set (similarly for other datasets), so a \textsc{BD-GRPO} batch over-represents hard examples relative to a random batch without any hand-crafted schedule.
Fig.~\ref{fig:rl_main} shows healthy training throughout RL: entropy stays bounded, KL-to-reference remains small, hard-set accuracy rises substantially while easy-set accuracy is retained.

The upsampling strategy appears to multiply per-step rollout cost, but Fig.~\ref{fig:rl_wallclock} shows the opposite: \textsc{BD-GRPO} reaches its peak performance $3$-$5\times$ faster than vanilla \textsc{B-GRPO} on all datasets while matching or exceeding its peak. The speedup comes from two sources: (i) every retained group provides a high-quality gradient signal, yielding more stable update directions than \textsc{B-GRPO}'s noisy mix of informative and zero-advantage groups; so \textsc{BD-GRPO} needs far fewer total optimizer steps to peak; and (ii) rollouts within a step are independent and dispatched concurrently across GPUs via vLLM's batched inference, whereas optimizer steps must run sequentially. In addition, \textsc{B-GRPO}'s class-balanced batches do not translate to a balanced \emph{effective} signal: zero-advantage groups are unevenly distributed across classes and batches, causing the model to oscillate between over-predicting each class, further slowing convergence. Apdx.~\ref{app:config} gives RL %
setting details.

RL learnings on hard cases are incorporated back into the SOP, deferred to Apdx. \ref{app:rl2rules} for brevity.

\vspace{-0.075in}
\section{Experiments}
\label{sec:experiments}

\vspace{-0.075in}
\subsection{Experiment Setup}
\begin{table*}[!t]
\vspace{-0.1in}
\centering
\caption{Classification (macro-F1, BAcc) and explanation (LLM Judge, NLI) performance of prompt-based (top) and fine-tuned (bottom) methods; \hlbest{best} and \hlrunup{runner-up} in green and orange. \spo substantially outperforms CoT; SFT improves classification but falls short in explanation, which SOP-grounded R-SFT significantly boosts; \method is best overall across datasets and most metrics, and the only method \textbf{strong on both}.
}
\label{tab:main}
\vspace{-0.1in}
\small
\setlength{\fboxsep}{2pt}%
\resizebox{\linewidth}{!}{%
\begin{tabular}{llrr|rr||rr|rr||rr|rr}
\toprule
 &  & \multicolumn{4}{c}{ContractNLI} & \multicolumn{4}{c}{ICLR Review} & \multicolumn{4}{c}{MIMIC Readmission} \\
\cmidrule(lr){3-6}\cmidrule(lr){7-10}\cmidrule(lr){11-14}
 Method & Explanation & mF1 & BAcc & Judge & NLI & mF1 & BAcc & Judge & NLI & mF1 & BAcc & Judge & NLI \\
\midrule
CoT prompt & reasoning & 0.743 & 0.757 & 4.31 & 0.576 & 0.789 & 0.786 & 3.18 & 0.160 & 0.517 & 0.528 & 3.58 & 0.045 \\
\textbf{SPO4SOP (I)} & SOP + reasoning & 0.790 & 0.789 & 4.33 & \hlrunup{\textbf{0.578}} & 0.795 & 0.815 & 3.32 & \hlrunup{\textbf{0.200}} & 0.597 & 0.593 & 3.65 & \hlbest{\textbf{0.178}} \\
\cmidrule(lr){1-14}
SFT (Cls. Head) & saliency & \hlrunup{\textbf{0.847}} & \hlrunup{\textbf{0.865}} & 3.43 & 0.243 & \hlbest{\textbf{0.825}} & \hlbest{\textbf{0.832}} & 2.97 & 0.112 & \hlrunup{\textbf{0.630}} & 0.618 & 3.36 & 0.051 \\
\textbf{R-SFT (I+II)} & SOP + reasoning & 0.817 & 0.828 & \hlrunup{\textbf{4.39}} & \hlrunup{\textbf{0.578}} & 0.821 & 0.821 & \hlrunup{\textbf{3.40}} & 0.196 & 0.620 & \hlrunup{\textbf{0.623}} & \hlrunup{\textbf{4.29}} & \hlrunup{\textbf{0.057}} \\
\textbf{eXTC (I+II+III)} & SOP + reasoning & \hlbest{\textbf{0.851}} & \hlbest{\textbf{0.878}} & \hlbest{\textbf{4.48}} & \hlbest{\textbf{0.599}} & \hlbest{\textbf{0.825}} & \hlrunup{\textbf{0.829}} & \hlbest{\textbf{3.44}} & \hlbest{\textbf{0.217}} & \hlbest{\textbf{0.652}} & \hlbest{\textbf{0.642}} & \hlbest{\textbf{4.31}} & \hlrunup{\textbf{0.057}} \\
\bottomrule
\end{tabular}
}
\vspace{-0.175in}
\end{table*}

\begin{table}[t]
\centering
\caption{\looseness=-1 \textbf{PO ablation} with Qwen3-4B (top) and gpt-4.1-mini (bottom) classifiers. \spo outperforms zero-shot CoT and matches or exceeds MIPROv2 and GEPA while producing a structured decision set of rules.}
\label{tab:po_ablation}
\setlength{\tabcolsep}{4pt}
\vspace{-0.1in}
\scalebox{0.8}{
\begin{tabular*}{1.25\linewidth}{@{\extracolsep{\fill}}lrrrrrr}
\toprule
 & \multicolumn{2}{c}{ContractNLI} & \multicolumn{2}{c}{ICLR} & \multicolumn{2}{c}{MIMIC} \\
\cmidrule(lr){2-3}\cmidrule(lr){4-5}\cmidrule(lr){6-7}
Method & mF1 & BAcc & mF1 & BAcc & mF1 & BAcc \\
\midrule
\multicolumn{7}{l}{\textit{Classifier: Qwen3-4B}} \\
CoT & 0.668 & 0.661 & 0.707 & 0.662 & 0.517 & 0.521 \\
\textbf{SPO4SOP} & 0.640 & 0.652 & 0.681 & 0.735 & 0.561 & 0.563 \\
\midrule
\multicolumn{7}{l}{\textit{Classifier: gpt-4.1-mini}} \\
CoT & 0.743 & 0.757 & 0.789 & 0.786 & 0.517 & 0.528 \\
MIPROv2 & 0.753 & 0.767 & 0.779 & 0.808 & 0.541 & 0.568 \\
GEPA & 0.790 & 0.786 & 0.790 & 0.783 & 0.450 & 0.499 \\
\textbf{SPO4SOP} & \textbf{0.790} & \textbf{0.789} & \textbf{0.795} & \textbf{0.815} & \textbf{0.597} & \textbf{0.593} \\
\bottomrule
\end{tabular*}
}
\vspace{-0.03in}
\end{table}

\vspace{-0.05in}

\noindent
\noindent\textbf{Datasets:~}
\looseness=-1 We use publicly available
datasets from three domains: legal contracts, academic peer reviews, and hospital readmissions. Prior work has studied these prediction tasks~\cite{shaham-etal-2022-scrolls,kang-etal-2018-dataset,jiang2023nyutron} but focuses on label prediction; our setting (Problem~1) additionally scores each explanation against per-example evidence, which requires 1) long, information-rich inputs whose labels require reasoning, and 2) per-example ground-truth evidence. All three datasets satisfy both, with 2185--3209 average tokens and instance-level evidence, unlike simpler tasks that do not require in-depth reasoning, such as AG News~\cite{NIPS2015_250cf8b5} and IMDB~\cite{maas-etal-2011-learning}, with only 52 and 295 average tokens.
Table \ref{tab:dataset} and Apdx.~\ref{app:data} give details. While designed for complex text, \method generalizes well to LIAR, a standard short-text benchmark (Apdx.~\ref{app:liar}).

\noindent
\noindent\textbf{Models and Configurations:~}
We organize the baselines into two: prompt-based and fine-tuned.
\cotho employs zero-shot Chain-of-Thought prompting \cite{DBLP:conf/nips/Wei0SBIXCLZ22}, while \mipro \cite{opsahl2024optimizing} and \gepa \cite{agrawal2026gepa} are two hard prompt optimization (PO) alternatives to our \spo, producing discrete but unstructured PO.
SFT uses a classification head
to fine-tune on the labels (only). In contrast, our Reasoning Distillation (Stage II) fine-tunes using both SOP-grounded teacher reasoning and the labels.

Student backbone is Qwen3-4B~\cite{DBLP:journals/corr/abs-2505-09388}, a compact LM used for all SFT and RL methods. gpt-4.1-mini is the classifier for all prompt-based methods and the teacher generating R-SFT reasoning traces.
All fine-tuning runs were carried out on 4$\times$ NVIDIA H100 80GB GPUs.
For Cls-head SFT with no native reasoning, we use post-hoc gradient$\times$activation input saliency from Captum~\cite{kokhlikyan2020captum} as its explanation, scored with the same metrics.
Setup for PO, SFT, and RL are detailed in Apdx.~\ref{app:config}.

\noindent
\noindent\textbf{Metrics:~}
We measure classification performance by \textbf{macro-F1} and \textbf{balanced accuracy}.
Explanation quality is assessed against ground-truth evidence using \textbf{NLI} and \textbf{LLM-judge} scores. NLI scores use MiniCheck-Flan-T5-Large \cite{tang2024minicheck}, %
a fact-checking model that scores per-sentence entailment of an explanation against the evidence. LLM judge (gemini-2.5-flash-lite) follows G-Eval \cite{liu2023geval}: for each (evidence, explanation) pair, it returns a 1--5 Likert score weighted by token-level output probabilities.
Details are in Apdx.~\ref{app:metric}.

\vspace{-0.1in}
\subsection{Main Results}
\label{ssec:main_results}
\vspace{-0.05in}

Table \ref{tab:main} presents results across datasets and metrics.
Basic zero-shot CoT prompting baseline considerably underperforms across tasks, demonstrating the necessity of training the language models for complex text classification. Our \textbf{\spo substantially outperforms CoT} in terms of both classification and explanation performance, despite being trained only on a small \textit{subset} of training data.

\textbf{SFT improves classification over prompting} by using \textit{all} available training data for fine-tuning, but \textbf{it deteriorates notably in explanation quality}, even against basic CoT, because saliency-extracted explanations are fragmented and confined to input spans, underscoring the necessity for reasoning-based fine-tuning.
Our \textbf{SOP-grounded reasoning distillation significantly boosts explanation quality} over end-to-end SFT, by combining Stage I (learned rulebook) and Stage II (R-SFT), while maintaining robust classification performance.

\textbf{\method achieves the best overall performance across datasets and majority of metrics} by further boosting R-SFT model through Stage III (RL with self-guided curriculum).
The improvements in classification and explanation performance justify \method's modular design, where the rulebook grounds R-SFT and RL further boosts the result.
Paired-bootstrap tests for every stage increment are given in Apdx.~\ref{app:significance}; from CoT to \method, all six contrasts across classification performance (mF1) and explanation quality (Judge score) are significant.

\vspace{-0.05in}
\subsection{Ablation Analysis}
\label{ssec:ablation}
\vspace{-0.05in}

 Table \ref{tab:po_ablation} compares \spo, one of our main contributions, with two existing hard PO approaches that do not produce structured rulebooks with stand-alone rules. \spo consistently performs better across datasets, with especially large gains over MIPROv2 on ContractNLI and GEPA on MIMIC.
With Qwen3-4B as the classifier (top block), CoT prompting generally performs worse than with gpt-4.1-mini, and \spo improves it only on MIMIC, suggesting that a 4B model is not equipped to optimize prompts for such complex inputs. \textbf{A capable teacher in Stages~I and II is therefore essential} for obtaining a strong small model.\clubpenalty=0

Table~\ref{tab:rl_ablation} shows that the learned SOP benefits both later stages. With all other settings fixed, distilling from the SOP-guided teacher yields higher macro-F1 and NLI on all three datasets, and this advantage persists after RL. \textbf{Stage~I is therefore an essential part of \method}, which RL alone cannot replace. RL improves the student with and without the SOP, so Stages~I and III are complementary. Both components of \textsc{BD-GRPO} also contribute. Dynamic filtering improves both macro-F1 and NLI on every dataset, and class balancing is critical on the imbalanced MIMIC, where RL without it nearly collapses to the majority class. An auxiliary LLM-judge reward mainly raises the Judge score, with no clear macro-F1 gain (Apdx.~\ref{app:aux}).

\begin{table}[t]
\centering
\caption{\textbf{Ablation of the SOP} (Stage II, top) and \textbf{RL design} (Stage III, bottom); w/o SOP: CoT teacher without the rulebook; \textbf{best} of each block in bold. \textbf{SOP improves both mF1 and NLI at both stages on all datasets}; dynamic filtering improves both on every dataset, and class balancing prevents collapse on MIMIC.}
\label{tab:rl_ablation}
\setlength{\tabcolsep}{3pt}
\vspace{-0.1in}
\scalebox{0.8}{
\begin{tabular*}{1.25\linewidth}{@{\extracolsep{\fill}}lrrrrrr}
\toprule
 & \multicolumn{2}{c}{ContractNLI} & \multicolumn{2}{c}{ICLR} & \multicolumn{2}{c}{MIMIC} \\
\cmidrule(lr){2-3}\cmidrule(lr){4-5}\cmidrule(lr){6-7}
Method & mF1 & NLI & mF1 & NLI & mF1 & NLI \\
\midrule
\textbf{R-SFT (II)} & \textbf{0.817} & \textbf{0.578} & \textbf{0.821} & \textbf{0.196} & \textbf{0.620} & \textbf{0.057} \\
\quad w/o SOP & 0.798 & 0.563 & 0.803 & 0.163 & 0.567 & 0.046 \\
\cmidrule(lr){1-7}
\textbf{\textsc{BD-GRPO} (III)} & 0.851 & \textbf{0.599} & \textbf{0.825} & \textbf{0.217} & \textbf{0.652} & 0.057 \\
\quad w/o SOP & 0.846 & 0.563 & 0.817 & 0.177 & 0.607 & 0.049 \\
\quad w/o balancing & \textbf{0.857} & 0.561 & 0.823 & 0.198 & 0.586 & \textbf{0.080}$^*$ \\
\quad w/o dyn.\ filter & 0.841 & 0.595 & 0.817 & 0.189 & 0.645 & 0.054 \\
\bottomrule
\end{tabular*}
}
\par\vspace{1pt}{\raggedright\footnotesize $^*$Collapsed to 6\% readmission predictions (true rate 22\%); its NLI benefits from mostly negative, easier-to-ground outputs.\par}
\end{table}
\vspace{-0.125in}
\subsection{Case Study}
\label{ssec:study}
\vspace{-0.05in}

Fig. \ref{fig:crownjewel} presents a case from MIMIC, where (a) shows one of the rules in the learned SOP (Apdx.~\ref{app:rules-spo4sop}) that predicts \texttt{readmit} under the trigger pattern ``unstable cancer'' \textit{with the exception} of ``no treatment started''.
In (b) right, \method predicts \texttt{no\_readmit}, based on the reasoning ``biopsy with results pending'' but with ``no indication of malignancy''. In contrast, in (b) left, CoT latches on surface information of ``pending biopsy'' to incorrectly predict \texttt{readmit}.
Apdx.~\ref{app:cases} gives more cases.\clubpenalty=0

\looseness=-1 After teacher distillation, \method \textbf{no longer} includes the SOP in its prompt. Its continued adherence to the learned rules indicates that this knowledge has been internalized through distillation (II), and that RL (III) remains faithful to those rules during exploration (Apdx.~\ref{app:agreement}).

\vspace{-0.065in}
\section{Conclusion}
\label{sec:conclusion}
\vspace{-0.065in}

We introduce \method, an effective text classifier that provides both local reasoning and a global, modular rulebook. By strategically uniting \textit{structured} prompt optimization, \textit{reasoning-grounded} distillation, and \textit{targeted} reinforcement learning, \method captures the best of both worlds: it overall outperforms traditional fine-tuning and prompt-optimization baselines while simultaneously offering high-quality, auditable interpretability.

\section*{Limitations}

Our work focused on general text corpora, leveraging and further enhancing LLM capabilities for structured prompt optimization, reasoning, and reinforcement learning. However, we recognize that many real-world corpora are inherently multimodal and contain information beyond plain text, such as images, charts, and structured numerical tables, which were not explicitly addressed in this work.
Stages I and II rely on proprietary LLMs; reproducing our results requires comparable access. We used temperature 0 where possible to encourage deterministic outputs, but environmental factors may introduce minor variation.

\section*{Ethical Considerations}
Our approach inherits ethical considerations associated with LLMs, including potential biases in pretrained models, privacy concerns, and risk of misleading or unfaithful generated rationales. Although structured reasoning is designed to facilitate interpretability, in high-stakes applications, outputs should  be used with appropriate human oversight and careful evaluation across diverse populations. Our use of MIMIC-IV adheres to the PhysioNet credentialed data use agreement, including its privacy requirements; we do not redistribute any record-level data or models trained on it (Apdx.~\ref{app:data}).

\noindent
\textbf{AI assistance:} We used Claude and ChatGPT for code drafting, debugging, and editorial polishing of this manuscript, with author oversight. All scientific design, experimental protocols, and analyses were performed by the authors.

\bibliography{custom,ref}

\clearpage
\newpage

\appendix
\section*{Appendix}
\label{sec:appendix}
\renewcommand{\dbltopfraction}{0.95}\renewcommand{\dblfloatpagefraction}{0.8}
\renewcommand{\topfraction}{0.95}\renewcommand{\bottomfraction}{0.8}\renewcommand{\textfraction}{0.05}\renewcommand{\floatpagefraction}{0.8}
\setcounter{dbltopnumber}{3}\setcounter{topnumber}{3}\setcounter{bottomnumber}{2}\setcounter{totalnumber}{5}

\section{Extended Related Work}
\label{app:related}

\begin{table*}[!tb]
\centering
\footnotesize
\setlength{\tabcolsep}{3pt}
\renewcommand{\arraystretch}{1.15}
\caption{Related work comparison. \textbf{Columns}: SPO = structured PO; RL$\rightarrow$R = rule mining from RL rollouts; \textbf{Global} \cmark{} requires a decomposable structured artifact (rule list / decision tree / program) inspectable without running the model; \textbf{Instance} \cmark{} requires a per-example reasoning trace grounded in that artifact (CoT / applied-rule trace / decision path). \textbf{Symbols}: \cmark{} = property fully satisfied; blank = absent; \pmark{} (SPO column only) marks methods whose rules are structured but concatenated into one prompt or routed per input rather than applied as an independently-composable rule set.
}
\vspace{-0.1in}
\resizebox{\textwidth}{!}{%
\begin{tabular}{@{}lllccccccccc@{}}
\toprule
\multirow{2}{*}{\textbf{Paper}} & \multirow{2}{*}{\textbf{Task}} & \multirow{2}{*}{\textbf{Venue}} & \multicolumn{5}{c}{\textbf{Method}} & \multicolumn{3}{c}{\textbf{Capability}} & \multirow{2}{*}{\textbf{Total}} \\
\cmidrule(lr){4-8} \cmidrule(lr){9-11}
 & & & PO & SPO & SFT & RL & RL$\rightarrow$R & Global & Instance & Infer. eff. & \\
\midrule
\textbf{Ours}                        & Class.   & ---            & \cmark & \cmark & \cmark & \cmark & \cmark & \cmark & \cmark & \cmark & \textbf{8/8} \\
\midrule
EvolveR~\cite{DBLP:journals/corr/abs-2510-16079}         & Agent    & arXiv'25-10    & \cmark & \pmark & \cmark & \cmark & \cmark &  &  & \cmark & 6/8 \\
MAC~\cite{thareja2026mac}            & Class.   & arXiv'26-03    & \cmark & \pmark & \cmark &  &  & \cmark &  & \cmark & 5/8 \\
DeepSeek-R1~\cite{DBLP:journals/nature/GuoYZSWZXZMBZY025} & Reason.  & Nature'25      &  &  & \cmark & \cmark &  &  & \cmark & \cmark & 4/8 \\
LSP~\cite{DBLP:conf/iclr/WangSYRHD25}               & Class.   & ICLR'25        & \cmark & \pmark &  &  &  & \cmark & \cmark &  & 4/8 \\
SkillRL~\cite{xia2026skillrl}        & Agent    & arXiv'26-02    &  &  &  & \cmark & \cmark &  &  & \cmark & 3/8 \\
RIMRULE~\cite{gao2026rimrule}        & Agent    & arXiv'26-01    & \cmark & \pmark &  &  &  & \cmark &  &  & 3/8 \\
SALMON~\cite{sun2024salmon}          & Align.   & ICLR'24        &  &  & \cmark & \cmark &  &  &  &  & 2/8 \\
GEPA~\cite{agrawal2026gepa}          & General  & ICLR'26        & \cmark &  &  &  &  &  &  &  & 1/8 \\
MIPROv2~\cite{opsahl2024optimizing}  & General  & NeurIPS'24     & \cmark &  &  &  &  &  &  &  & 1/8 \\
\bottomrule
\end{tabular}%
}
\label{tab:related-work}

\end{table*}

Much of the related work is already referenced in the main text, and here we provide an extended discussion. Table~\ref{tab:related-work} outlines directly-relevant prior work along five \emph{methodology} axes (PO, SPO, SFT, RL, RL$\rightarrow$R rule mining) and three \emph{capability} axes (global/model-level explanation, local/instance-level explanation, inference efficiency). Each existing work covers only a strict subset of these axes; \emph{no} prior system jointly delivers global and instance-level interpretability for complex text classification with a deployment-efficient student.

\paragraph{Structured prompt optimization.}
Prompt optimization has evolved from token-level rewriting~\cite{DBLP:conf/emnlp/PryzantI0L0023} to text-as-parameter approaches that treat the prompt as a learnable artifact under verbalized feedback~\cite{xiao2024verbalized,DBLP:journals/nature/YuksekgonulBBLLHGZ25}. Representative methods MIPROv2~\cite{opsahl2024optimizing} and GEPA~\cite{agrawal2026gepa} optimize over unstructured instruction strings; a parallel line targets \emph{structured} rule-sets, including LSP~\cite{DBLP:conf/iclr/WangSYRHD25} and MAC~\cite{thareja2026mac}. These structured baselines still select a per-input rule \emph{subset} and concatenate the selected rules during inference (tree traversal in LSP, agentic re-selection in MAC). In contrast, our SOP is a \emph{composable} rule set, applied uniformly at inference, and atomic, stand-alone, and fully present---so the rule set itself could serve as a global explanation.

\paragraph{Reasoning distillation and RL on reasoning.}
Chain-of-thought supervision~\cite{DBLP:conf/nips/Wei0SBIXCLZ22} underpins reasoning-distillation pipelines that fine-tune students on teacher rationales~\cite{henrichsen2025reasoning,deng2024adapt}. Beyond distillation, group-based policy optimization (GRPO, \citealp{DBLP:journals/corr/abs-2402-03300}; DAPO, \citealp{DBLP:journals/corr/abs-2503-14476}) and verifiable-reward RL (DeepSeek-R1, \citealp{DBLP:journals/nature/GuoYZSWZXZMBZY025}) improve reasoning quality beyond the teacher; a separate line uses learned reward models to align LLMs with high-level principles (SALMON, \citealp{sun2024salmon}), targeting alignment rather than task-specific reasoning. Our Stage~II conditions the teacher on the SOP so that the student inherits a rule-grounded chain of thought; Stage~III uses \textsc{BD-GRPO} to concentrate learning on the hard examples the SOP cannot resolve.

\paragraph{Interpretable text classification.}
Traditional interpretability is post-hoc: feature-attribution methods such as LIME~\cite{DBLP:conf/kdd/Ribeiro0G16} and SHAP~\cite{lundberg2017shap}, and gradient/saliency-style scoring~\cite{lyu2024towards} explain a black-box decision token-by-token but offer no global model-level account~\cite{DBLP:journals/natmi/Rudin19}. A newer line pursues either global rulebooks (LSP, MAC; see above) or instance-level rationales such as prototype surrogates~\cite{wei2025protosure}; few provide both global and local explanations within a single deployment-efficient student model.

\paragraph{Rule mining from RL rollouts.}
A nascent line recovers interpretable artifacts from trained policies: distilling neural agents into programs~\cite{kohler2025distillpolicy} and tool-using agents that consolidate failure traces via MDL-guided rules~\cite{gao2026rimrule}. Closest are agent-side rulebook learners EvolveR~\cite{DBLP:journals/corr/abs-2510-16079} and SkillRL~\cite{xia2026skillrl}. In contrast, eXTC loops RL-mined rules back into the global SOP for text classification.

\section{\method Details}

This appendix expands on \method's Stage~I structured prompt optimization with \spo (Apdx.~\ref{app:spo}) and Stage~III reinforcement learning with \textsc{BD-GRPO} (Apdx.~\ref{app:rl}).

\subsection{Proposed Structured Prompt Optimization: Details}
\label{app:spo}

We elaborate on \spo (Stage~I): an overall flow diagram of the algorithm (Apdx.~\ref{app:flow}), the rule-subset selection step that picks the inference-time working set (Apdx.~\ref{app:subselect}), and a scalability analysis showing the procedure scales linearly in the number of candidate rules (Apdx.~\ref{app:efficiency}).

\subsubsection{Flow Diagram}
\label{app:flow}

Fig.~\ref{fig:flow} illustrates the main components and the workflow of \spo algorithm designed for \method's Stage I.

\begin{figure*}[!th]
 \centering 
\includegraphics[width=0.95\textwidth]{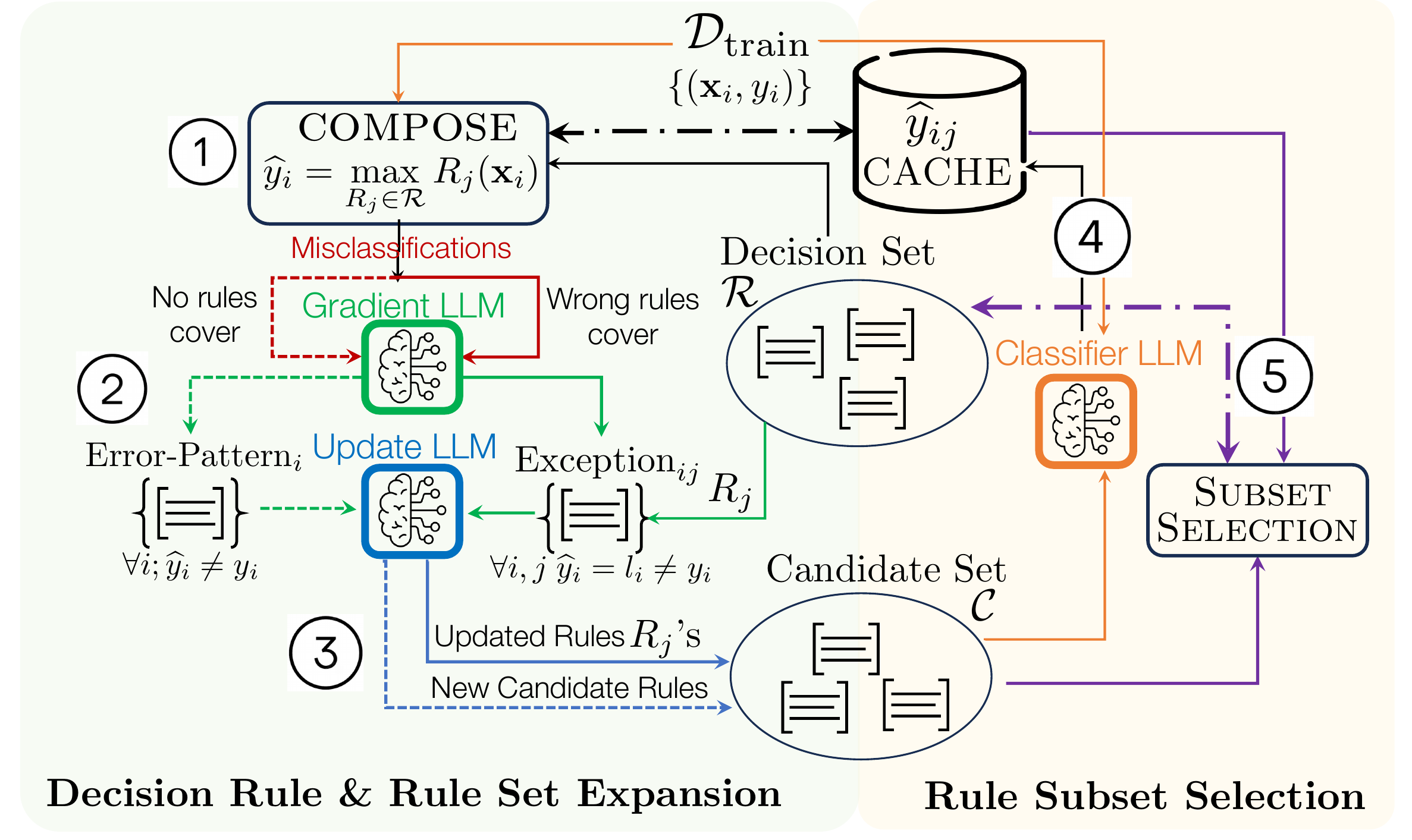}
 \caption{\label{fig:flow} \textbf{\method's Stage I:} \spo for structured prompt (i.e. decision (rule) set $\mR$) optimization consists of two main primitives: \textbf{(1) Decision Rule \& Rule Set Expansion} employs Gradient and Update LLMs on misclassifications to obtain updated and new candidate rules, and \textbf{(2) Rule Subset Selection} employs Beam-Search over the persistent pool $\mP_t$ of all accumulated candidate rules to maximize macro-F1 on the validation set. Thanks to composable decisions, individual rule predictions by Classifier LLM are cached and composed efficiently to evaluate potential rule subsets---
 allowing search over a combinatorial space while keeping the number of LLM calls linear in the number of rules, rather than in the exponentially many possible rule \textit{subsets}.}
 \end{figure*}

\subsubsection{Rule Subset Selection: Details}
\label{app:subselect}

Given the persistent pool $\mP_t := \mP_{t-1} \cup \mC$ accumulating all candidate rules generated through iteration $t$ (including the previous pool $\mP_{t-1}$ and the (updated as well as new) candidates $\mC$ from this iteration), the next step is to select a small subset $\mS \subseteq \mP_t$ that achieves low classification error on the held-out validation data $\mD_{\text{val}}$.

To compute classification error, we
first obtain individual predictions $\widehat{y}_{ij} := R_j(\bx_i)$ of each rule on each sample. Importantly, these standalone \textit{predictions are composed easily and efficiently for any subset of rules} via max-pooling to reach the final prediction---thanks to the \textit{composability} of decision sets (Definition \ref{def:composition}).
Decision composition is a key notion in our design that fosters efficiency; we simply obtain predictions for the newly added candidate rules $\mC$ \textit{only}, as they are generated, and \textbf{cache} these to be used in composition during subset search (Line~\ref{lin:classify}), while the $\widehat{y}_{ij}$ cache readily contains predictions for the rules in $\mP_{t-1}$ from previous iterations.

The persistent pool $\mP_t := \mP_{t-1} \cup \mC$ (Line~\ref{lin:expand}), which accumulates all candidates generated through iteration $t$, is used within a beam search for subset selection (Line~\ref{lin:subselect}). Letting $|\mP_t|=m$ and the binary \textit{variable} $z_j\in \{0,1\}$ be the indicator for the $j$'th rule being selected, Eq.~\ref{eq:opt} depicts the optimization model.
In the objective function, $l(\widehat{\mathbf{y}}, \mathbf{y})$ denotes the task-specific loss (macro-$F_1$ in our experiments), where $\mathbf{y}$ and $\widehat{\mathbf{y}}$ denote the label vector and the composed prediction vector over the dataset, respectively.
The second term encourages the sparsity of rule selection, where $\lambda$ controls the trade-off between the classification loss and the subset size. A small set promotes interpretability and encourages semantic diversity among the rules.

\begin{equation}\label{eq:opt}
\begin{aligned}
    & \min_{z_j\in\{0,1\}, j \in [m]}\; l(\widehat{\mathbf{y}}, \mathbf{y}) + \lambda \sum_j z_j \quad \text{s.t.}\\
    &  \quad \widehat{y}_i = \begin{cases}
    \displaystyle \max_{j\in[m]}\{z_j\cdot \widehat{y}_{ij}\}, & \text{\small binary},\\[1ex]
    \displaystyle \arg\max_{j\in[m]}\{z_j\cdot p(\widehat{y}_{ij})\}, & \text{\small multi-class}
    \end{cases}
\end{aligned}
\end{equation}

We remark that the beam search does not necessarily return the optimal subset if the beam size is not large enough. Arguably, we do not need the global optimum in practice as it might lead to overfitting (especially given that the validation set is typically small in prompt-optimization settings).

For the optimal subset returned after the last iteration (Line~\ref{lin:return}), one can in principle linearize the formulation in Eq.~\ref{eq:opt} and solve it exactly. In practice, we use the beam-search approximation described above throughout.

\subsubsection{Scalability Analysis}
\label{app:efficiency}

The key computational bottleneck of \spo is the rule subset search primitive, not the rule \& rule set expansion primitive (see Fig.~\ref{fig:flow}).
The latter's LLM calls for the ``gradient'' ``updates'' are negligible since those ($i$) are over a mini-batch of samples $\Dbatch$, and ($ii$) involve \grad calls only for misclassified samples in the batch, while ($iii$) \up is called once per rule in Line~\ref{lin:fup-rev} and only once in Line~\ref{lin:fup-new}. 

This section shows that rule subset search via beam search is scalable, in fact, linear in the number of candidates (as opposed to exponential in the number of possible subsets). We restate the lemma from \S\ref{ssec:spo} and prove it here.

\begin{lemma}[Scalable subset search, restated]
\spo's rule subset search scales linearly with the number of candidate rules $m$, rather than with the exponential number of possible rule subsets. This is enabled by decision composition (Definition~\ref{def:composition}), which requires only a linear number of LLM calls in $m$ to obtain decisions for arbitrary rule subsets.
\end{lemma}

\begin{proof}
As a decision set, our structured prompt design enables \textit{composable} label assignment and consequently, scalable search.
In particular, our decision set
 consists of {standalone} rules, which allows us to predict each sample by each rule in the set \textit{independently}. This allows independent LLM calls, using a prompt composed of a single rule at a time (meanwhile reducing context length), to obtain predictions for $n$ samples by $m'$ additional candidate rules in $\mC$, which we \textit{cache}. Note that the predictions by the existing $(m$$-$$m')$ rules in $\mP_{t-1}$ are already cached in previous iterations. Overall, subset selection necessitates $nm$ \texttt{Classifier LLM} calls \textit{prior to} beam search (Line~\ref{lin:classify}, Algorithm~\ref{algo:spo4sop}).

For beam search with width $b$, expansion at depth $d$ yields $b(m-d)$ subset candidates to evaluate (the objective in Eq.~\ref{eq:opt}). Having cached existing $b$ subsets' predictions of $n$ samples, adding a new rule to a subset takes $O(1)$ time per sample to update its prediction---by picking the maximum of the current/cached prediction and the new rule's prediction (also cached).
Over $K$ steps of the beam search, $O(bKm)$ subsets are evaluated, \textit{without any additional LLM calls}, saving $O(bKmn)$ LLM calls over the course of the search.

In effect, decision composition allows us to search in a combinatorially large space while necessitating
the number of LLM calls to be only \textbf{linear} in the number of rules $m$---and \textbf{not} in the exponential number of rule subsets.
\end{proof}

In comparison, prior prompt optimization algorithms \cite{DBLP:conf/emnlp/PryzantI0L0023,DBLP:conf/iclr/WangLW0LZJXH24,xiao2024verbalized,DBLP:journals/nature/YuksekgonulBBLLHGZ25} can only evaluate entire prompts, as they lack compositionality. In theory, searching the same prompt space would require these existing approaches $O(2^m n)$ LLM calls, compared to $O(mn)$ by \spo.

\subsection{Reinforcement Learning: Details}
\label{app:rl}

This section provides additional analysis of the RL stage: per-dataset training dynamics (Apdx.~\ref{app:rldyn}) and the procedure for folding RL-derived findings back into the SOP (Apdx.~\ref{app:rl2rules}).

\subsubsection{RL Dynamics}
\label{app:rldyn}

\looseness=-1 At the start of RL, Table~\ref{tab:hardness_bucket_full} extends the hardness $\times$ rollout-correctness cross-tab of Table~\ref{tab:hardness_bucket} (ICLR Review in the main body) to all three datasets, characterizing the informative band before any training step is taken.

\begin{table}[!t]
\centering
\caption{Full version of Table~\ref{tab:hardness_bucket}: rollout-outcome distribution per hard/easy group at the start of RL on all three datasets ($G$=8).}
\label{tab:hardness_bucket_full}
\setlength{\tabcolsep}{4pt}
\setlength{\fboxsep}{2pt}%
\small
\begin{tabular}{llrrr|r}
\toprule
 &  & \multicolumn{3}{c}{Rollouts correct (\%)} & \\
\cmidrule(lr){3-5}
Dataset & Diff. & 0/G & \colorbox{yellow!30}{$1$\textendash $7/G$} & G/G & \% of band \\
\midrule
ContractNLI & hard & 44.1 & \colorbox{yellow!30}{49.4} & 6.5 & \colorbox{yellow!30}{29.9} \\
 & easy & 0.3 & \colorbox{yellow!30}{14.9} & 84.8 & \colorbox{yellow!30}{70.1} \\
 & all & 5.3 & \colorbox{yellow!30}{18.8} & 75.8 & (11.4\% nat.) \\
\midrule
ICLR Review & hard & 54.0 & \colorbox{yellow!30}{43.6} & 2.4 & \colorbox{yellow!30}{29.5} \\
 & easy & 0.7 & \colorbox{yellow!30}{18.1} & 81.2 & \colorbox{yellow!30}{70.5} \\
 & all & 8.6 & \colorbox{yellow!30}{21.9} & 69.5 & (14.8\% nat.) \\
\midrule
MIMIC Readmit & hard & 65.3 & \colorbox{yellow!30}{33.7} & 1.0 & \colorbox{yellow!30}{24.4} \\
 & easy & 1.4 & \colorbox{yellow!30}{27.1} & 71.5 & \colorbox{yellow!30}{75.6} \\
 & all & 14.6 & \colorbox{yellow!30}{28.5} & 57.0 & (20.6\% nat.) \\
\bottomrule
\end{tabular}
\end{table}

Over the course of training, \S\ref{ssec:rl} (Fig.~\ref{fig:rl_main}) shows ContractNLI as a representative example; Figs.~\ref{fig:rl-extra-iclr} and~\ref{fig:rl-extra-mimic} report the same four training-dynamics curves (actor entropy, KL to reference, hard-pool correct rate, easy-pool correct rate) over rollout rounds for the other two datasets. All RL training curves are exponentially-moving-average (EMA) smoothed with $\alpha{=}0.99$ for the hit-rate metrics and $\alpha{=}0.95$ for entropy and KL.

\begin{figure}[!t]
  \centering
  \includegraphics[width=\linewidth]{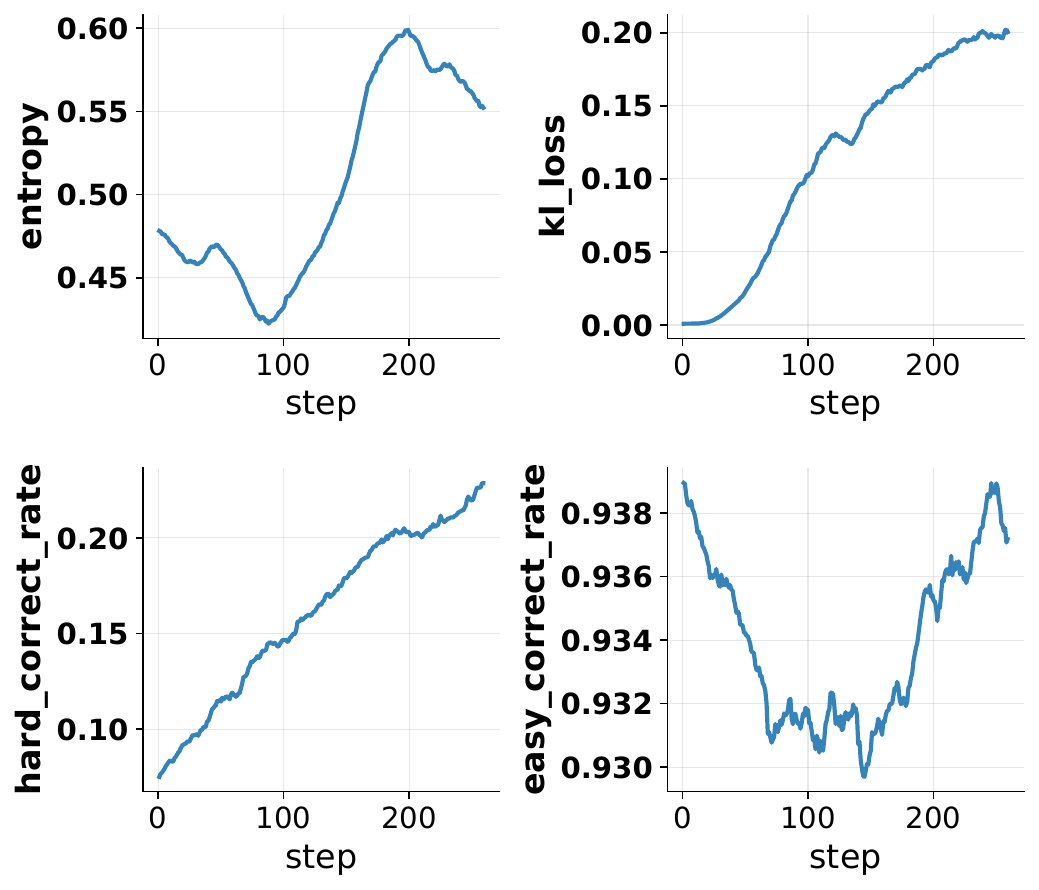}
  \caption{RL training dynamics on ICLR Review, parallel to Fig.~\ref{fig:rl_main}. Entropy stays bounded (briefly dips around step 80 then rises to $\sim$0.60, no collapse), KL grows but remains small ($\leq$0.20), \textit{hard-pool correct rate rises steadily from $\sim$0.08 to $\sim$0.23 ($\sim$3$\times$)}, easy-pool correct rate preserved ($\sim$0.93).}
  \label{fig:rl-extra-iclr}
\end{figure}

\begin{figure}[!t]
  \centering
  \includegraphics[width=\linewidth]{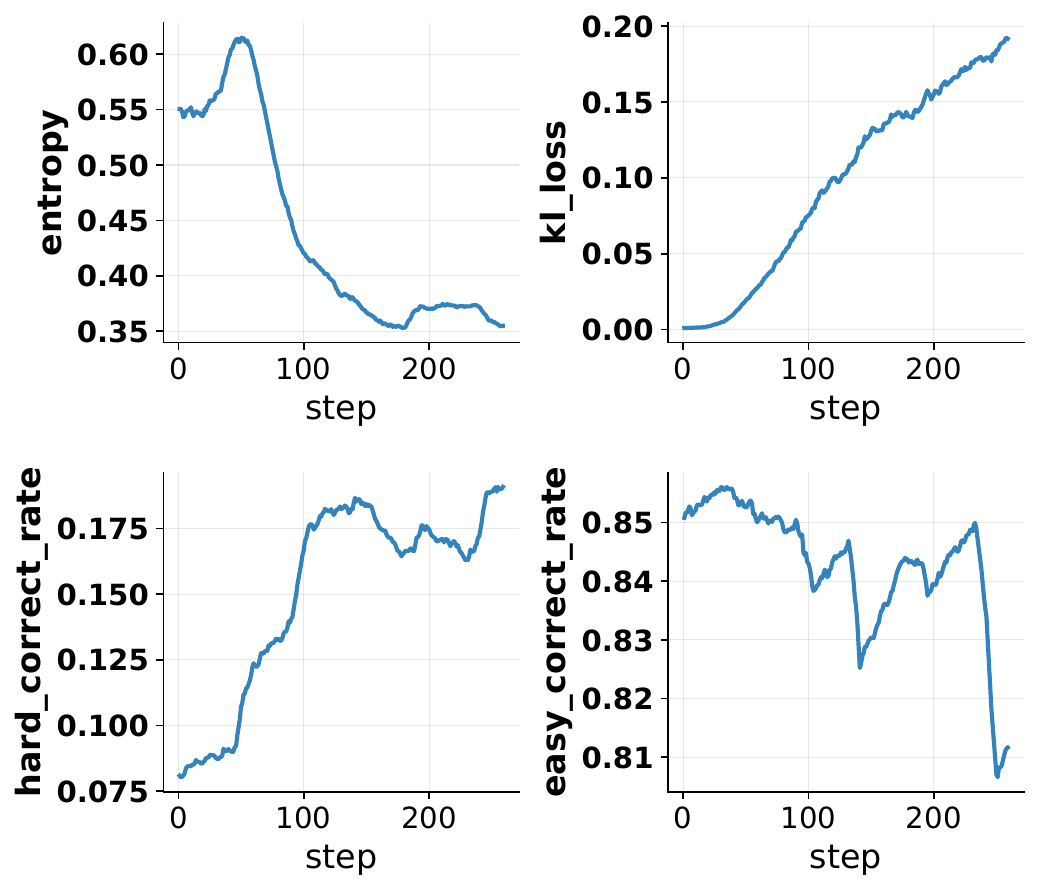}
  \caption{RL training dynamics on MIMIC Readmission, parallel to Fig.~\ref{fig:rl_main}. Entropy decreases from $\sim$0.55 to $\sim$0.36 but stays well above collapse, KL remains small ($\leq$0.20), \textit{hard-pool correct rate roughly doubles from $\sim$0.08 to $\sim$0.19} then plateaus, easy-pool correct rate largely retained ($\sim$0.81--0.85) with brief late-stage dips---consistent with the same hard-band learning pattern.}
  \label{fig:rl-extra-mimic}
\end{figure}

\subsubsection{Rulebook Revision}
\label{app:rl2rules}

The SOP learned by \spo (Apdx.~\ref{app:spo}) covers most inputs but leaves a residual error band of \emph{hard} inputs where the post-RL model reasons correctly while a teacher applying the current SOP does not. The Rulebook Revision pipeline distills these residual strategies into candidate SOP rules in five stages: (1)~collect correct RL rollouts on hard inputs; (2)~cluster them by reasoning strategy via three independent LLM rounds followed by a merge pass; (3)~assign each rollout to one cluster; (4)~synthesize one candidate rule per (cluster,~target~label); (5)~deduplicate near-equivalent candidates within each cluster.

All five stages are LLM-driven; the per-stage prompts are released with our code (Apdx.~\ref{ssec:rl2r-prompts}). Stage~2 asks the LLM to describe an analytical \emph{strategy} (the reasoning method, not the input topic): three independent rounds plus a merge pass reduce single-round variance. Stage~3 bins each rollout into one strategy cluster. Stage~4 is the core synthesis call: per cluster, the LLM is shown \emph{paired} traces (teacher-incorrect vs.\ RL-correct on the same inputs) with a stratified mix of positive (gold $=$ target label) and negative (gold $\neq$ target label) examples; a SKIP gate suppresses candidates with insufficient signal for the target label. Stage~5 is a pairwise LLM equivalence judge that collapses near-duplicate candidates within a cluster.

\paragraph{Selection on \emph{val\_hard}.} From the candidate rule set, we apply the same beam-search SubsetSelect logic used by \spo (Apdx.~\ref{app:subselect}) to choose the final additions, but with the validation pool restricted to \emph{val\_hard} --- the val examples on which the teacher applying the existing SOP fails all $M{=}4$ attempts. Restricting selection to \emph{val\_hard} avoids reusing the train-hard signal that the RL model already exploited during training, and isolates the marginal value of the new rules on examples the existing SOP cannot solve.

\paragraph{Evaluation on \emph{test\_hard}.} We evaluate the selected rules (composed per Definition~\ref{def:composition}) on the analogously-defined \emph{test\_hard} subset, against the CoT baseline (zero-shot teacher, no rules). For ICLR Review, the teacher with the existing SOP over-predicts \texttt{reject} on hard inputs (91\% of val\_hard has gold $=$ \texttt{accept}); we therefore synthesize label-\texttt{accept} rules, with \texttt{reject} as the default label of the new rule set. Table~\ref{tab:rl2r-testhard} shows that the new rules consistently improve both macro-F1 and balanced accuracy over CoT across all three datasets.

\begin{table}[t]
\centering
\small
\setlength{\tabcolsep}{3pt}
\resizebox{\linewidth}{!}{%
\begin{tabular}{lcccccc}
\toprule
 & \multicolumn{2}{c}{ContractNLI} & \multicolumn{2}{c}{ICLR Review} & \multicolumn{2}{c}{MIMIC Readmit} \\
\cmidrule(lr){2-3}\cmidrule(lr){4-5}\cmidrule(lr){6-7}
Method & mF1 & BAcc & mF1 & BAcc & mF1 & BAcc \\
\midrule
CoT (no rules)      & 0.143 & 0.150 & 0.312 & 0.255 & 0.321 & 0.437 \\
RL$\to$R\,rules  & \textbf{0.255} & \textbf{0.335} & \textbf{0.468} & \textbf{0.495} & \textbf{0.355} & \textbf{0.505} \\
\bottomrule
\end{tabular}
}
\caption{Evaluation of the RL$\to$R-selected rules on \emph{test\_hard} (test examples on which the teacher with the existing SOP fails all $M{=}4$ attempts). Per-dataset sizes are $n{=}235$ / $195$ / $108$ for ContractNLI / ICLR Review / MIMIC Readmit. CoT = zero-shot teacher (no rules).}
\label{tab:rl2r-testhard}
\end{table}

\section{Prompt Templates and Instructions}
\label{app:prompts}

\subsection{Generic prompt templates}

These templates are shared across all three datasets, with per-dataset
placeholders (\texttt{\{task\_framing\}}, \texttt{\{input\_tag\}}, \texttt{\{input\_noun\}},
\texttt{\{labels\}}) substituted from Table~\ref{tab:appendix-placeholders}. The
\emph{Reasoning Prompt with Rules} drives the R-SFT teacher (uses the learned
SOP); the \emph{Reasoning Prompt without Rules} is used by the zero-shot CoT
baseline, the R-SFT student, and during RL; and the \emph{Classification (Label-Only) Prompt} is used when
SFT runs with a classification head instead of generative decoding.

\begin{table*}[!tbp]
\centering\small
\caption{Per-dataset values for the placeholders used in the generic prompt templates above.}
\label{tab:appendix-placeholders}
\begin{tabular}{@{}llll@{}}
\toprule
Placeholder & ContractNLI & ICLR Review & MIMIC Readmission \\
\midrule
\texttt{\{input\_tag\}} & \texttt{<CONTRACT\_HYPOTHESIS\_PAIR>} & \texttt{<REVIEWER\_COMMENTS>} & \texttt{<DISCHARGE\_SUMMARY>} \\
\texttt{\{input\_noun\}} & contract text and the hypothesis & reviewer comments & discharge summary \\
\texttt{\{labels\}} & not\_mentioned / entailment / contradiction & accept / reject & readmitted / not\_readmitted \\
\bottomrule
\end{tabular}
\end{table*}

\begin{prompttemplatebox}{Reasoning Prompt with Rules (R-SFT teacher)}
\footnotesize
Below is a rulebook of patterns relevant to the task. Treat the rulebook
as internal guidance --- it shapes what to look for in the
\texttt{\{input\_noun\}}, but should not be cited in your reasoning.

\texttt{<RULES>}\\
\texttt{\{rulebook\}}\\
\texttt{</RULES>}

\texttt{\{input\_tag\}}\\
\texttt{\{text\}}\\
\texttt{</\ldots>}

When writing your reasoning, analyze the \texttt{\{input\_noun\}} directly.
Do not name or enumerate rules in your reasoning.

Return exactly in this format:\\
\texttt{REASONING:}\\
your reasoning

\texttt{LABEL: <one of \{labels\}>}

Please think step by step.
\end{prompttemplatebox}

\begin{prompttemplatebox}{Reasoning Prompt without Rules (CoT baseline / R-SFT student / RL)}
\footnotesize
\texttt{\{task\_framing\}}

\texttt{\{input\_tag\}}\\
\texttt{\{text\}}\\
\texttt{</\ldots>}

Analyze the \texttt{\{input\_noun\}} directly to decide the label.

Return exactly in this format:\\
\texttt{REASONING:}\\
your reasoning

\texttt{LABEL: <one of \{labels\}>}

Please think step by step.
\end{prompttemplatebox}

\begin{prompttemplatebox}{Classification (Label-Only) Prompt (SFT classification head)}
\footnotesize
\texttt{\{task\_framing\}}

\texttt{\{input\_tag\}}\\
\texttt{\{text\}}\\
\texttt{</\ldots>}

Return exactly:\\
\texttt{LABEL: <one of \{labels\}>}
\end{prompttemplatebox}

\subsection{\spo meta-prompts}

\paragraph{Classifier LLM.}
\label{ssec:fmodel}
The Classifier LLM is invoked once per (sample, rule) pair: it sees a single
rule at a time and returns whether the rule fires (label \texttt{\{RULE\_LABEL\}})
or abstains (label \texttt{\{ABSTAIN\}}). Task framing is supplied via the
system prompt (the same one used by the SFT teacher above). For multi-class
tasks (e.g.\ \textsc{ContractNLI}, 3 classes), the last sentence of the output
spec asks the model to first conclude any class, then output the rule's target
label only if its conclusion matches it.

\begin{prompttemplatebox}{Per-Rule Classifier (Classifier LLM)}
\footnotesize
Here is the rule you want to check:\\
\texttt{<RULE>}\\
\texttt{\{rule\_text\}}\\
\texttt{</RULE>}

\texttt{<REPORT>}\\
\texttt{\{report\}}\\
\texttt{</REPORT>}

Provide your detailed reasoning under a header exactly written as
\texttt{REASONING:}.
Then provide your final prediction under a header exactly written as
\texttt{FINAL PREDICTION:}.
Use only one of these values for the final prediction:
\texttt{\{RULE\_LABEL\}} (the rule applies) or \texttt{\{ABSTAIN\}} (the rule does not apply
or you cannot tell).
Please think step by step:
\end{prompttemplatebox}

\paragraph{Gradient LLM.}
\label{ssec:fgrad}
The Gradient LLM is invoked in two modes, depending on the error type. Both
share the structured-field output convention; the bullet-pointed
\texttt{exceptions} / \texttt{points} are then forwarded to the Update LLM.

\begin{prompttemplatebox}[breakable]{Exception Generator --- False Coverage (Gradient LLM)}
\footnotesize
You are an expert at \texttt{\{CLASSIFICATION\_TASK\}} and error analysis.

Review the rule and report where the existing rule may be too broad;
propose exceptions to restrict it.

\texttt{<RULE>}\\
\texttt{\{RULE\}}\\
\texttt{</RULE>}

\texttt{<REPORT>}\\
\texttt{\{REPORT\}}\\
\texttt{</REPORT>}

The incorrect prediction made by the model: \texttt{\{PREDICTION\}}

The correct label for this instance: \texttt{\{LABEL\}}

In your response, provide a detailed explanation in the analysis field,
then list the exceptions in the exceptions field. Explain why the model's
prediction was incorrect and what type of error or misunderstanding may
have occurred.
\end{prompttemplatebox}

\begin{prompttemplatebox}[breakable]{Error-Pattern Generator --- Blind Spot (Gradient LLM)}
\footnotesize
You are an expert at \texttt{\{CLASSIFICATION\_TASK\}} and error pattern recognition.

There are errors on the analysis below. Identify the underlying error
pattern that the current rules fail to capture.

\texttt{<REPORT>}\\
\texttt{\{REPORT\}}\\
\texttt{</REPORT>}

\texttt{<RELEVANT\_RULES>}\\
\texttt{\{MATCHING\_RULES\}}\\
\texttt{</RELEVANT\_RULES>}

The incorrect prediction made by the model: \texttt{\{PREDICTION\}}

The correct label for this instance: \texttt{\{LABEL\}}

In your response, provide a diagnostic summary in the summary field and
key points in the points field. Explain why the model's prediction was
incorrect and what type of error or misunderstanding may have occurred.
\end{prompttemplatebox}

\paragraph{Update LLM: rule revision (false coverage).}
\label{ssec:fup1}
The Update LLM consumes the Gradient LLM's exception bullets and emits a
revised rule, formatted with a \texttt{<RULE\_NAME>} /
\texttt{<RULE\_DESCRIPTION>} schema used elsewhere in the pipeline.

\begin{prompttemplatebox}[breakable]{Rule Update --- False Coverage (Update LLM)}
\footnotesize
Your task is to generate a new rule to avoid overly broad coverage mistakes
by augmenting the exceptions and examples in an existing rule. Follow
these steps:

Review the provided existing rule:

\texttt{<EXISTING\_RULE>}\\
\texttt{\{RULE\}}\\
\texttt{</EXISTING\_RULE>}

\texttt{<EXCEPTION\_NOTES>}\\
\texttt{\{EXCEPTIONS\}}\\
\texttt{</EXCEPTION\_NOTES>}

Keep the rule label unchanged as \texttt{\{RULE\_LABEL\}}.

Identify the core pattern or principle that led to the mistake described
in the existing rule.

Determine any exceptions where the rule should not apply.

Formulate a new rule using the following structure:

\texttt{<RULE\_NAME>[Concise, descriptive new name]</RULE\_NAME>}

\texttt{<RULE\_DESCRIPTION>}\\
\texttt{Trigger Pattern:}\\
\texttt{[Keep original pattern in the existing rule, clarify if needed]}

\texttt{Exceptions:}\\
\texttt{[Keep original exceptions]}\\
\texttt{[Add new exceptions to restrict over-broad coverage]}

\texttt{Examples}\\
\texttt{[Keep existing examples]}\\
\texttt{[Add new examples if needed]:}\\
\texttt{Source text: [Relevant source text if any]}\\
\texttt{Wrong:~~~[Example of violation]}\\
\texttt{Correct: [Example of compliance]}\\
\texttt{</RULE\_DESCRIPTION>}

Ensure the rule has a new name.
\end{prompttemplatebox}

\paragraph{Update LLM: new rule synthesis (blind spot).}
\label{ssec:fup2}
When no rule fires on a misclassified example, the Update LLM is prompted to
synthesize one or more brand-new rules from the Gradient LLM's error
patterns, grouped by target label.

\begin{prompttemplatebox}[breakable]{New Rule Synthesis --- Blind Spot (Update LLM)}
\footnotesize
You are an expert at analyzing patterns and developing precise rules for
\texttt{\{CLASSIFICATION\_TASK\}}. Your task is to systematically evaluate error
patterns and create a clear, actionable and strict rule to prevent
similar mistakes in future classifications.

The goal is to identify the truly exceptional instances, not just average
ones.

Generate rules with label \texttt{\{TARGET\_LABEL\}}.

Return at most \texttt{\{MAX\_NEW\_RULES\}} new rule(s).

Review the existing rules:

\texttt{<EXISTING\_RULES>}\\
\texttt{\{RULES\}}\\
\texttt{</EXISTING\_RULES>}

Error patterns from misclassifications:

\texttt{<ERROR\_PATTERNS>}\\
\texttt{\{ERROR\_PATTERNS\}}\\
\texttt{</ERROR\_PATTERNS>}

First, conduct a structured error analysis:\\
\texttt{<ERROR\_ANALYSIS>}\\
\texttt{1.~Context Assessment}\\
\texttt{~~~What was the intended behavior?}\\
\texttt{~~~What actually happened?}\\
\texttt{~~~What key factors contributed?}\\
\texttt{2.~Rule Evaluation}\\
\texttt{~~~Do existing rules partially cover this?}\\
\texttt{~~~What aspects are unique?}\\
\texttt{~~~How can we make the rule robust?}\\
\texttt{</ERROR\_ANALYSIS>}

Provide the error analysis in the error\_analysis field.

Then, formulate a new rule using this structure:

\texttt{<RULE\_NAME>[Concise, descriptive name that clearly identifies the pattern]</RULE\_NAME>}

\texttt{<RULE\_DESCRIPTION>}\\
\texttt{Trigger Pattern: [Clear description of when this rule applies, with 2-3}\\
\texttt{~~~~~~~~~~~~~~~~~~specific indicators that should trigger the rule]}

\texttt{Exceptions: [Specific cases where the rule should not be applied, even if}\\
\texttt{~~~~~~~~~~~~~the pattern appears to match]}

\texttt{Example [One clear example]:}\\
\texttt{Source text: [Relevant source text if any]}\\
\texttt{Wrong:~~~[Example of violation]}\\
\texttt{Correct: [Example of compliance]}\\
\texttt{</RULE\_DESCRIPTION>}

Please analyze the error examples thoroughly and formulate a new rule that
would prevent similar classification errors in the future.
\end{prompttemplatebox}

\subsection{Rulebook Revision meta-prompts}
\label{ssec:rl2r-prompts}

The Rulebook Revision pipeline (Apdx.~\ref{app:rl2rules}) distills new
candidate SOP rules from post-RL behavior on hard inputs. Five LLM-driven
stages are used (strategy taxonomy discovery, taxonomy merge, rollout
classification, per-cluster rule synthesis, and rule equivalence judging); the
per-stage prompts are released with our code.

\section{Experiment Setup Details}
\label{app:setup}

This appendix gives the full setup behind \S\ref{sec:experiments}: dataset construction and splits (Apdx.~\ref{app:data}), model and training configurations for all baselines and \method variants (Apdx.~\ref{app:config}), evaluation metrics (Apdx.~\ref{app:metric}), and LLM-judge details with full prompts (Apdx.~\ref{app:judge}).

\subsection{Dataset Description}
\label{app:data}

\textbf{\textsf{ContractNLI}}~\cite{DBLP:conf/emnlp/KoreedaM21}.\footnote{\url{https://stanfordnlp.github.io/contract-nli/}} We use the Stanford ContractNLI corpus, comprising 10{,}319 (contract, hypothesis) pairs over 607 non-disclosure agreements. Each instance carries a human-labeled evidence span---the contiguous portion of the contract that supports the label---used as ground truth for our NLI- and judge-based explanation metrics. We use the official train/dev/test split (7{,}191 / 1{,}037 / 2{,}091).

\textbf{\textsf{ICLR\_Reviews}.}\footnote{\url{https://huggingface.co/datasets/Daoze/ReviewRebuttal}} Sourced from the Re$^2$ corpus \cite{zhang2025re2}, which aggregates peer-review records across 24 conferences and 21 workshops on OpenReview (including NeurIPS, ACL, and others). We restrict to ICLR main-conference rows (workshops, tiny-papers, and blog tracks are excluded) and keep all venue-matched papers in the official Re$^2$ train/test split, additionally carving out an in-distribution validation set from train. The resulting split is 8{,}552 / 1{,}000 / 655 for the binary \emph{accept}/\emph{reject} task. Year coverage is inherited from Re$^2$: train/val span ICLR 2017, 2018, 2020--2023; test spans ICLR 2020--2023 only (no test rows from 2017/2018; ICLR 2019 appears in Re$^2$ as workshops only and is filtered out; ICLR 2024 is absent from Re$^2$). Each paper's evidence is its official meta-review summary; reviewer texts (input) and meta-review (evidence) are disjoint.

\textbf{\textsf{MIMIC\_Readmission}.}\footnote{\url{https://physionet.org/content/mimiciv/3.1/}} Built from MIMIC-IV~v3.1 \cite{Johnson_2023} discharge summaries. We subsample 10{,}000 admissions via patient-level iterative multi-label stratification \cite{sechidis2011stratification} jointly balancing the binary readmission label and the discharge service, yielding an 8{,}000 / 1{,}000 / 1{,}000 train/val/test split with no patient overlap across splits. Evidence is a structured per-admission summary aligned to the LACE \cite{walraven2010lace} schema; comorbidities are mapped from \texttt{diagnoses\_icd} via Quan-2005 Charlson coding \cite{quan2005coding}, and the ED component is proxied by 6-month inpatient admission counts (MIMIC-IV's ED module is unaligned with our cohort). Output format and per-field formulas follow LACE conventions; see the released code for exact mapping.

\paragraph{Licenses, intended use, and de-identification.}
All datasets and tools are used in compliance with their licenses, for non-commercial research only, consistent with their intended use. \textbf{MIMIC-IV} requires PhysioNet credentialed access under the PhysioNet Credentialed Health Data License v1.5.0; we obtained this access via the standard CITI training and Data Use Agreement and do not redistribute any record-level data. For the same reason, we withhold our MIMIC checkpoints, as models fine-tuned on clinical notes can memorize and reproduce patient text. MIMIC-IV is de-identified per HIPAA Safe Harbor by the dataset providers; we did no additional re-identification or aggregation steps. \textbf{ContractNLI} uses publicly redacted NDAs and \textbf{ICLR Review} uses public OpenReview metadata, neither of which contains personal identifiers beyond the original public release. \textbf{Qwen3-4B} is released under Apache-2.0; Captum, MiniCheck, and DSPy under permissive open-source licenses (BSD/MIT). Closed-source LLMs (gpt-4.1-mini, gpt-4o-mini, gemini-2.5-flash-lite) are accessed via OpenAI / Google APIs in compliance with their terms of service.

\paragraph{Train/val splits per method.} \spo and PO baselines (MIPROv2, GEPA) were run on the same 180-row class-balanced PO-train and 400-row natural-ratio PO-val, across each dataset. Both PO-train and PO-val are subsets of the training set; this is to prevent pollution of the validation set in stages II and III. R-SFT and RL use the full training and validation sets (see Table~\ref{tab:dataset}).

\subsection{Model Configuration Details}
\label{app:config}

Tables~\ref{tab:hparams_ft} and~\ref{tab:hparams_po} summarize the hyperparameters of all fine-tuning (FT) and prompt-optimization (PO) methods, respectively. All FT runs share the same student backbone (Qwen3-4B) and hardware (4$\times$~H100 80GB). SFT (Cls.~Head) is trained full-parameter with the HuggingFace \texttt{Trainer} (\texttt{AutoModelForSequenceClassification}); R-SFT uses LlamaFactory \cite{DBLP:conf/acl/ZhengZZYL24}; RL uses verl \cite{sheng2024hybridflow} with our custom sampler that supports upsampling and group-decoupled auxiliary reward. For all FT methods we retain the checkpoint with the best validation macro-$F_1$ and report all numbers on the held-out test set under this selection. All FT methods use the AdamW optimizer with mixed-precision (bf16) training and a single random seed (42). For \textsc{BD-GRPO}, the per-class quota $n_c$ is set to $\lfloor B/C \rfloor$ with a rotating remainder when $C \nmid B$, where $B{=}16$ is the effective batch size and $C \in \{2, 3\}$ is the number of classes (e.g., $8$/$8$ for binary, $6$/$5$/$5$ for 3-class).

\begin{table*}[!tbp]
\centering
\caption{Fine-tuning method hyperparameters. RL uses verl GRPO with dynamic oversample-then-filter on top of an R-SFT-merged initialization. RL ablation variants (class balance, dynamic filter, auxiliary reward) and their specific configurations are detailed in Table~\ref{tab:hparams_rl_ablation}.}
\label{tab:hparams_ft}
\setlength{\tabcolsep}{3pt}
\small
\vspace{-0.05in}
\begin{tabular}{lccc}
\toprule
 & \textbf{SFT (Cls. Head)} & \textbf{R-SFT} & \textbf{RL (\textsc{eXTC})} \\
\midrule
Backbone        & Qwen3-4B & Qwen3-4B & Qwen3-4B (R-SFT init) \\
Fine-tuning     & Full-parameter & LoRA $r{=}64, \alpha{=}128$ & Full-parameter \\
LoRA target     & --- & all linear (q/k/v/o + MLP) & --- \\
Epochs / Steps  & 4 epochs & 2 epochs & 260 steps \\
Learning rate   & $5\mathrm{e}{-}6$ (10\% warmup) & $2\mathrm{e}{-}4 \to 2\mathrm{e}{-}5$ (cos) & $1\mathrm{e}{-}6$ (10\% warmup) \\
Eff.\ batch     & 32 & 32 & 16 \\
Max input len.  & 8192 & 8192 & 8192 \\
Max output len. & --- & --- & 1024 \\
\bottomrule
\end{tabular}
\end{table*}

\begin{table*}[!tbp]
\centering
\caption{Prompt-optimization method hyperparameters. All PO methods use the same 180/400 PO-train/PO-val subset (Apdx.~\ref{app:data}); CoT requires no optimization.}
\label{tab:hparams_po}
\setlength{\tabcolsep}{3pt}
\small
\vspace{-0.05in}
\begin{tabular}{lcccc}
\toprule
 & \textbf{CoT} & \textbf{MIPROv2} & \textbf{GEPA} & \textbf{\textsc{SPO4SOP}} \\
\midrule
Task model        & gpt-4.1-mini & gpt-4.1-mini & gpt-4.1-mini & gpt-4.1-mini \\
\quad temperature & 0 & 0 & 0 & 0 \\
Optimizer model   & --- & gpt-5.4-mini & gpt-5.4-mini & gpt-5.4-mini \\
\quad temperature & --- & 1 & 1 & 1 \\
\quad reasoning   & --- & low & low & low \\
Train / Val       & --- & 180 / 400 & 180 / 400 & 180 / 400 \\
Budget            & --- & auto=medium & auto=heavy & $T{=}6$, batch$=$30, beam$=$15 \\
Subset constraint & --- & --- & --- & $K{=}8$, $\lambda{=}1.0$ \\
Selection         & --- & Bayesian & Pareto + Reflection & Reflection + Beam \\
\bottomrule
\end{tabular}
\end{table*}

\paragraph{PO baseline implementation and metric choice.}
We invoke MIPROv2 and GEPA through the official DSPy framework \cite{journals/corr/abs-2310-03714}. Both methods are designed around per-sample scalar feedback and cannot natively optimize aggregate metrics such as macro-$F_1$ or balanced accuracy. We therefore use the binary \texttt{correctness} signal (1 if the predicted label matches the gold label, 0 otherwise), equivalent to optimizing accuracy. Neither baseline is aware of class imbalance during prompt search, partially explaining their weaker performance on the most label-imbalanced dataset, MIMIC Readmission, reported in Table~\ref{tab:po_ablation}.

\paragraph{Cls-head SFT saliency extraction.}
For the Cls-head SFT baseline (Qwen3-4B with a classification head, no native reasoning), we generate a token-level pseudo-explanation post-hoc with Captum's \texttt{LayerGradientXActivation} \cite{kokhlikyan2020captum} (akin to Integrated Gradients~\cite{pmlr-v70-sundararajan17a}) on the input-embedding layer, targeting the predicted-class logit. This attribution gives a signed, magnitude-aware contribution per token, isolating the input tokens that pull the model toward its prediction---the natural candidates for an evidence-like span set. Per-token scores are summed over the hidden dimension and ReLU-clamped (keeping only tokens that support the predicted label); high-saliency spans are then picked by greedy non-maximum suppression over $W{=}20$-token windows until the kept tokens reach the median length of the R-SFT reasoning on that dataset, so that the judges see explanations of comparable length. Selected spans are decoded back to text, yielding the pseudo-reasoning string scored by the NLI and LLM-judge metrics.

\begin{table*}[!tbp]
\centering
\caption{RL ablation variant hyperparameters, matching the rows of Tables~\ref{tab:rl_ablation} and~\ref{tab:aux} (performance); w/o SOP rows use the \textsc{BD-GRPO} (III) settings. $^\dagger$ RL runs on ContractNLI override the entropy coefficient to $0.001$ for all variants to encourage exploration; ICLR Review and MIMIC use the defaults shown. $^\ddagger$ w/o dyn.\ filter (\textsc{B-GRPO} in Fig.~\ref{fig:rl_wallclock}) is allotted $3\times$ more optimizer steps because, lacking dynamic filtering, each of its steps consumes the same nominal rollout budget but yields a much noisier (mostly zero-advantage) gradient; the longer schedule gives it the wall-clock budget needed to peak (see Fig.~\ref{fig:rl_wallclock}).}
\label{tab:hparams_rl_ablation}
\setlength{\tabcolsep}{3pt}
\small
\vspace{-0.05in}
\begin{tabular}{lcccc}
\toprule
 & \textbf{w/o dyn.\ filter} & \textbf{w/o balancing} & \textbf{\textsc{BD-GRPO} (III)} & \textbf{\textsc{BD-GRPO}+aux} \\
\midrule
Init                  & R-SFT merged & R-SFT merged & R-SFT merged & R-SFT merged \\
Fine-tuning           & Full-param & Full-param & Full-param & Full-param \\
Learning rate         & $1\mathrm{e}{-}6$ & $1\mathrm{e}{-}6$ & $1\mathrm{e}{-}6$ & $1\mathrm{e}{-}6$ \\
LR warmup             & 10\% & 10\% & 10\% & 10\% \\
KL coef               & 0.001 & 0.001 & 0.001 & 0.001 \\
Rollout temp.         & 1.0 & 1.0 & 1.0 & 1.0 \\
Eff.\ batch           & 16 & 16 & 16 & 16 \\
Rollouts ($n$)        & 8 & 8 & 8 & 8 \\
Entropy coef$^\dagger$ & 0 & 0 & 0 & 0 \\
Training steps$^\ddagger$ & 780 & 260 & 260 & 260 \\
\midrule
Clip (low / high)     & 0.2 / 0.2 & 0.2 / 0.28 & 0.2 / 0.28 & 0.2 / 0.28 \\
Upsampling batch size & --- & 6$\times$ batch & 6$\times$ batch & 6$\times$ batch \\
Class-balanced batch  & \ding{51} & \ding{55} & \ding{51} & \ding{51} \\
Aux reward            & --- & --- & --- & LLM-judge ($\bx \to \br$), $\lambda{=}0.2$ \\
\bottomrule
\end{tabular}
\end{table*}

\subsection{Metrics}
\label{app:metric}

\noindent\textbf{Classification performance metrics.~}
Let $C$ be the number of classes and let $\mathrm{recall}_c$, $\mathrm{precision}_c$, $F_{1,c}$ denote the per-class recall, precision, and $F_1$. We report:
\begin{align*}
\text{macro-}F_1 &= \tfrac{1}{C} \textstyle\sum_{c=1}^{C} F_{1,c}, \\
\text{BAcc} &= \tfrac{1}{C} \textstyle\sum_{c=1}^{C} \mathrm{recall}_c.
\end{align*}
\looseness=-1 Both are unweighted means over classes, robust to label imbalance.

Predictions that fail to parse into a valid label (e.g., format violations, missing \texttt{LABEL:} field) are counted as wrong for the gold class, so an unparseable output contributes one false negative to its true class and zero true positives anywhere. Parse-failure rates are below 1\% across all runs.

\noindent\textbf{Explanation quality metrics.~}
For each test example we evaluate the model's reasoning trace $\br$ against the dataset-provided evidence $\bx_{\text{ev}}$ with two complementary metrics:
\begin{itemize}[noitemsep,topsep=0pt,leftmargin=1em]
\item \emph{NLI score}: per-sentence entailment with $\bx_{\text{ev}}$ as the premise and each sentence of $\br$ as the hypothesis, scored by MiniCheck-Flan-T5-Large \cite{tang2024minicheck}. We average the per-sentence entailment probability over sentences in $\br$.
\item \emph{LLM-judge score}: a G-Eval-style 1--5 Likert rating \cite{liu2023geval} of how well $\br$ is grounded in $\bx_{\text{ev}}$, returned by gemini-2.5-flash-lite.
\end{itemize}
For all metrics higher is better. NLI and LLM-judge are computed only on examples for which evidence is available; coverage is 100\% for ContractNLI / ICLR\_Reviews and~$\sim$95\% for MIMIC\_Readmission (admissions without LACE-derived evidence are excluded).

\subsection{LLM-judge details}
\label{app:judge}

Two G-Eval-style LLM-judges appear in our pipeline with distinct roles. \textbf{(i)~Evaluation judge (Groundedness):} scores reasoning $\br$ against the held-out evidence span $\bx_{\text{ev}}$; powers the explanation-quality metric reported in our results tables. \textbf{(ii)~Auxiliary reward judge (Faithfulness):} scores reasoning $\br$ against the model input $\bx$ (not $\bx_{\text{ev}}$) as a dense secondary reward during RL training in BD-GRPO + aux ablation (Table~\ref{tab:hparams_rl_ablation}). We describe each in turn below. The full prompts of both judges for each dataset are released with our code.

\paragraph{Evaluation judge: model.}
We use gemini-2.5-flash-lite \cite{DBLP:journals/corr/abs-2507-06261} accessed through the Google Enterprise Agent Platform (formerly Google Vertex AI), which exposes per-token log-probabilities required for G-Eval-style probability-weighted scoring. gemini-2.5-flash-lite is from a different model family than the task classifier (gpt-4.1-mini) and the student LLM (Qwen3-4B), which could avoid the same-family preference bias of LLM judges~\cite{panickssery2024llm}.

\paragraph{Evaluation judge: scoring.}
We extract a continuous score from the judge's output as the expected value $s = \sum_{i=1}^{5} i \cdot p(i)$ over the renormalized top-$k$ token-level probabilities of $\{$\texttt{1},\texttt{2},\texttt{3},\texttt{4},\texttt{5}$\}$ at the answer position. The same expected-value extraction is reused by the auxiliary judge below.

\paragraph{Evaluation judge: prompt.}
The Groundedness prompt follows the G-Eval layout: task description, evaluation criterion, evaluation steps, then the evidence and the reasoning. Its 1--5 scale is anchored on how the reasoning's conclusion (its predicted label) is tied to the evidence. A 5 ties the label to specific cited evidence (a contract section, a named reviewer and score, a clinical value, or a named rule) through a multi-step chain; a 3 justifies it only by general paraphrase; a 1 contains a claim the evidence refutes, or reaches no verdict. The template is the same across the three datasets except for the task description and the name of the evidence (e.g., ``the meta-review summary'' for ICLR Review).

\paragraph{Auxiliary judge: motivation and model.}
We use gpt-4o-mini~\cite{DBLP:journals/corr/abs-2410-21276} with a G-Eval-style \emph{Faithfulness} rubric \cite{liu2023geval} on the input--reasoning pair $(\bx, \br)$, \emph{not} the evidence--reasoning pair $(\bx_{\text{ev}}, \br)$: $\bx_{\text{ev}}$ is by construction a near-sufficient determinant of the label (e.g., the ICLR Review meta-review summary is strongly related to the final acceptance decision). Rewarding alignment with $\bx_{\text{ev}}$ would leak label-revealing text unavailable to baselines, whereas $\bx$ is already the model's inference-time context, so faithfulness to $\bx$ is leak-free. The auxiliary judge reuses the same expected-value extraction as the evaluation judge.

\paragraph{Auxiliary judge: prompt.}
The Faithfulness prompt has the same layout. It rates whether every claim in the reasoning can be traced to the input $\bx$, and penalizes misreadings, fabricated details, and external knowledge not present in $\bx$.

\paragraph{Auxiliary judge: call cost.}
With train batch $B{=}16$ prompts per step, rollout group size $G{=}8$, and $T{=}260$ training steps per run, the auxiliary judge is called $B \cdot G \cdot T = 16 \cdot 8 \cdot 260 = 33{,}280$ times per run. Across the three +aux runs (ContractNLI / ICLR Review / MIMIC), this totals $\sim$100k gpt-4o-mini calls.

\paragraph{Auxiliary judge: group-decoupled normalization in \textsc{BD-GRPO}.}
A naive addition of the auxiliary reward to the binary correctness reward $R^{\text{cor}}_i \in \{-1, +1\}$ would dilute the per-group sparsity that the \textsc{BD-GRPO} oversample-then-filter step explicitly preserves. \textsc{BD-GRPO} drops any rollout group whose correctness reward has $\mathrm{std}(\{R^{\text{cor}}_j\}_{j=1}^{G}) = 0$ because that group carries no contrastive learning signal; pre-mixing a dense 1--5 auxiliary score into the reward would silently let the filter keep groups that have no correctness signal at all. To avoid this we follow the GDPO group-reward-decoupled formulation~\cite{liu2026gdpo}: the auxiliary judge is invoked \emph{after} the \textsc{BD-GRPO} filter (so only correctness-signal-bearing groups are scored, halving API cost vs.\ pre-filter scoring), and the two reward streams are z-scored independently within each surviving group before being summed into a single advantage:
\begingroup
\setlength{\abovedisplayskip}{2pt}\setlength{\belowdisplayskip}{2pt}
\begin{align*}
\hat{A}^{\bullet}_i &= \frac{R^{\bullet}_i - \mathrm{mean}(\{R^{\bullet}_j\}_{j=1}^{G})}{\mathrm{std}(\{R^{\bullet}_j\}_{j=1}^{G}) + \varepsilon}, \quad \bullet \in \{\text{cor}, \text{aux}\}, \\
\hat{A}_i &= \hat{A}^{\text{cor}}_i + \lambda \cdot \hat{A}^{\text{aux}}_i, \qquad \lambda{=}0.2.
\end{align*}
\endgroup
Here mean and std are computed over the $G{=}8$ rollouts of the same prompt that survive the \textsc{BD-GRPO} filter; rollouts inapplicable to the auxiliary judge prompt (empty reasoning or empty input) receive $\hat{A}^{\text{aux}}_i{=}0$ and pass through unchanged. Because the two advantages are z-scored independently, the auxiliary contribution stays in the same unit-variance scale as the correctness advantage regardless of how the raw 1--5 Faithfulness scores cluster, and the $\lambda{=}0.2$ weight keeps its magnitude at roughly one-fifth of the main signal. This is small enough that \textsc{BD-GRPO}'s filter decision and the overall optimization direction remain driven by correctness, while still giving the model a smooth, dense secondary signal that shapes which correct rollouts (and which incorrect-but-faithful rollouts) are pushed up or down. Results are in Apdx.~\ref{app:aux}.

\begin{table*}[!tbp]
\centering
\caption{Stage-by-stage gains; paired-bootstrap two-sided $p$ in parentheses; \textbf{bold}: significant gain ($p<0.05$).}
\label{tab:significance}
\setlength{\tabcolsep}{4pt}
\small
\newcommand{\pval}[1]{{\scriptsize\textcolor{gray}{(#1)}}}
\begin{tabular}{lr@{\,}lr@{\,}lr@{\,}lr@{\,}lr@{\,}lr@{\,}l}
\toprule
 & \multicolumn{6}{c}{macro-F1} & \multicolumn{6}{c}{Judge} \\
\cmidrule(lr){2-7}\cmidrule(lr){8-13}
Stage & \multicolumn{2}{c}{ContractNLI} & \multicolumn{2}{c}{ICLR} & \multicolumn{2}{c}{MIMIC} & \multicolumn{2}{c}{ContractNLI} & \multicolumn{2}{c}{ICLR} & \multicolumn{2}{c}{MIMIC} \\
\midrule
Zero-shot CoT (baseline) & 0.743 & & 0.789 & & 0.517 & & 4.31 & & 3.18 & & 3.58 & \\
+ Stage I (rulebook) & \textbf{+0.047} & \pval{$<$.001} & +0.006 & \pval{.71} & \textbf{+0.080} & \pval{.001} & +0.02 & \pval{.37} & \textbf{+0.14} & \pval{.001} & \textbf{+0.07} & \pval{.013} \\
+ Stage II (distillation) & \textbf{+0.027} & \pval{.002} & \textbf{+0.026} & \pval{.047} & +0.023 & \pval{.16} & \textbf{+0.06} & \pval{$<$.001} & \textbf{+0.08} & \pval{.049} & \textbf{+0.64} & \pval{$<$.001} \\
+ Stage III (\textsc{BD-GRPO}) & \textbf{+0.034} & \pval{$<$.001} & +0.004 & \pval{.68} & \textbf{+0.032} & \pval{.024} & \textbf{+0.09} & \pval{$<$.001} & +0.04 & \pval{.12} & +0.02 & \pval{.09} \\
\midrule
\textbf{Overall (CoT$\to$\method)} & \textbf{+0.108} & \pval{$<$.001} & \textbf{+0.036} & \pval{.014} & \textbf{+0.135} & \pval{$<$.001} & \textbf{+0.17} & \pval{$<$.001} & \textbf{+0.26} & \pval{$<$.001} & \textbf{+0.73} & \pval{$<$.001} \\
\bottomrule
\end{tabular}
\end{table*}

\section{Additional Experiment Results}
\label{app:results}

This appendix presents results complementing \S\ref{sec:experiments}: results on LIAR (Apdx.~\ref{app:liar}), significance tests for every stage increment (Apdx.~\ref{app:significance}), how NLI and Judge scores move across stages (Apdx.~\ref{app:nli_judge}), the auxiliary judge reward (Apdx.~\ref{app:aux}), training and inference cost (Apdx.~\ref{app:cost}), agreement between the rulebook and the student (Apdx.~\ref{app:agreement}), the \spo validation trajectory across iterations, and additional case studies on ICLR Review and ContractNLI (Apdx.~\ref{app:cases}).

\subsection{Additional Benchmark: LIAR}
\label{app:liar}
To check that \method generalizes beyond our three domain-specific datasets, we also evaluate it on LIAR~\cite{wang-2017-liar}, a standard fact-checking benchmark that prompt optimization methods such as ProTeGi~\cite{DBLP:conf/emnlp/PryzantI0L0023} also use. Each example is a short political statement with the speaker's job title, state, and party and the context of the statement, rated by PolitiFact on a six-level truthfulness scale. LIAR-PLUS~\cite{alhindi-etal-2018-evidence} pairs each statement with the fact-checkers' justification, which serves as the ground-truth evidence $\bx_{\text{ev}}$ for our explanation metrics. We use the official splits, dropping training statements that are duplicated or also occur in val or test (10{,}243 / 1{,}284 / 1{,}283), and the binary setting of \citet{alhindi-etal-2018-evidence}: \emph{pants on fire}, \emph{false}, and \emph{barely true} map to \texttt{lie}, and the other three ratings to \texttt{not\_lie}.

The stage-by-stage gains of our main datasets carry over (Table~\ref{tab:liar}): \method improves on CoT by 0.085 mF1 and also raises the Judge score. \textbf{Most of the gain comes from Stage I: the learned rulebook} lifts mF1 from 0.568 to 0.644 and the NLI score from 0.204 to 0.300. The 4B student then matches and exceeds this gpt-4.1-mini classifier, reaching 0.647 after R-SFT and 0.653 after RL, the best classification of all methods, ahead of the Cls. Head (0.637). Stage III adds less here than on our main datasets, likely because a single short statement leaves little to reason about beyond what the rules already capture, unlike a long contract, review set, or discharge summary. For the same reason, LIAR suits saliency explanations: a short statement needs little reasoning, so the phrases that saliency highlights already read as a plausible explanation, and the Cls. Head receives the highest Judge score (3.56), though the lowest NLI score (0.120).

\begin{table}[!t]
\centering
\caption{Results on LIAR (test, $n{=}1{,}283$); \hlbest{best} and \hlrunup{runner-up} in green and orange.}
\label{tab:liar}
\small
\setlength{\fboxsep}{2pt}%
\begin{tabular*}{\linewidth}{@{\extracolsep{\fill}}lrrrr}
\toprule
Method & mF1 & BAcc & Judge & NLI \\
\midrule
CoT prompt & 0.568 & 0.618 & 3.17 & 0.204 \\
\textbf{SPO4SOP (I)} & 0.644 & 0.646 & 3.07 & \hlbest{\textbf{0.300}} \\
\midrule
SFT (Cls. Head) & 0.637 & 0.638 & \hlbest{\textbf{3.56}} & 0.120 \\
\textbf{R-SFT (I+II)} & \hlrunup{\textbf{0.647}} & \hlrunup{\textbf{0.647}} & 3.17 & \hlrunup{\textbf{0.286}} \\
\textbf{\method{} (I+II+III)} & \hlbest{\textbf{0.653}} & \hlbest{\textbf{0.655}} & \hlrunup{\textbf{3.21}} & 0.277 \\
\bottomrule
\end{tabular*}
\end{table}

\subsection{Stage-wise Significance}
\label{app:significance}

We ran paired bootstrap (10,000 resamples, per-example pairing, full test sets), giving every stage increment a significance test. All reported results come from a single run (one seed) per configuration. In Table~\ref{tab:significance}, the first row is the CoT baseline and each stage row is the change over the previous stage; the left block reports classification performance (macro-F1) and the right block explanation quality (LLM Judge).

Overall, from CoT to \method, all six contrasts across classification performance (mF1) and explanation quality (Judge score) are significant, indicating that the gains are substantive rather than marginal. Every stage contributes: all 18 stage increments in Table~\ref{tab:significance} are positive, and 12 of them are significant. Stages~I and III, which learn the rulebook and boost performance, significantly improve macro-F1 on ContractNLI and MIMIC. Stage~II, which transfers grounded reasoning from the teacher to the student, significantly raises the Judge score on all three datasets, and also macro-F1 on ContractNLI and ICLR.

\subsection{NLI and Judge Scores Across Stages}
\label{app:nli_judge}
The two explanation metrics agree on the overall picture: from CoT to \method, both the Judge and NLI scores rise on all three datasets, and every increase is significant (paired bootstrap, $p<0.001$; Table~\ref{tab:significance} for the Judge). The only stage where they clearly diverge is Stage II on MIMIC: the Judge score increases from 3.65 to 4.29, while NLI does not. The difference reflects what the metrics reward: NLI favors sentences that closely match the evidence, whereas the Judge evaluates whether the reasoning as a whole supports the predicted label with specific evidence. At this stage, the reasoning shifts from restating evidence toward using it to make a decision, which the Judge rewards.

On MIMIC, Stage I applies the rulebook explicitly, spelling out the risk factors checked by each rule. These closely match the evidence: although learned from labels alone, \spo recovers clinically meaningful rules such as ``Active Malignancy/Transplant Only When Disease Is Unstable or Intensively Managed'' (Table~\ref{rule:mimic_diagnostic_workup}), and Stage I achieves the highest NLI score on MIMIC (0.178). After distillation, the student has internalized these rules and no longer enumerates them. Its reasoning is 44\% shorter (1{,}162 vs.\ 2{,}069 characters) and more direct, which the Judge rates 0.64 higher.

\subsection{Auxiliary Judge Reward}
\label{app:aux}
Table~\ref{tab:aux} compares \textsc{BD-GRPO} with and without the auxiliary judge reward (Apdx.~\ref{app:judge}). The auxiliary reward raises the Judge score on all three datasets, most on ICLR Review (3.44 to 3.72), with similar macro-F1 on ICLR Review and MIMIC. \method keeps the correctness-only reward as its default, which needs no judge calls during training.

\begin{table}[!tb]
\centering
\caption{\textsc{BD-GRPO} with (\texttt{+aux}) and without the auxiliary judge reward; \textbf{best} in bold. ContractNLI +aux starts from an R-SFT model trained without the SOP.}
\label{tab:aux}
\small
\setlength{\tabcolsep}{3pt}
\begin{tabular*}{\linewidth}{@{\extracolsep{\fill}}lrrrrrr}
\toprule
 & \multicolumn{2}{c}{ContractNLI} & \multicolumn{2}{c}{ICLR} & \multicolumn{2}{c}{MIMIC} \\
\cmidrule(lr){2-3}\cmidrule(lr){4-5}\cmidrule(lr){6-7}
Method & mF1 & Judge & mF1 & Judge & mF1 & Judge \\
\midrule
\textsc{BD-GRPO} (III) & \textbf{0.851} & 4.48 & 0.825 & 3.44 & \textbf{0.652} & 4.31 \\
\textsc{BD-GRPO}+aux & 0.829 & \textbf{4.49} & \textbf{0.827} & \textbf{3.72} & 0.646 & \textbf{4.33} \\
\bottomrule
\end{tabular*}
\end{table}

\subsection{Computational Cost}
\label{app:cost}

\paragraph{Training cost.}
Table~\ref{tab:train_cost} reports the one-time costs. API rows report LLM-call dollar costs for prompt optimization and teacher sampling; GPU rows report 4$\times$H100 wall-clock hours.

Three take-aways follow. First, structure does not cost extra at the prompt-optimization stage: \spo sits in the same range as the unstructured optimizers GEPA and MIPROv2, because Lemma~\ref{lem:subset} (\S\ref{ssec:spo}) keeps the rule search linear in the number of rules rather than exponential.

\looseness=-1 Second, given the large data size and long input tokens (corpora of 15--26M tokens), the pipeline total, roughly \$39--\$50 and 3.1--6.3 GPU-hours, is a reasonable one-time cost. It provides three important gains: 1) best overall classification performance, 2) cheap and fast inference (Table~\ref{tab:infer_cost}), and 3) structured global and instance-level explanations.

Third, Cls-Head fine-tuning is indeed cheaper. However, it does not provide meaningful explanations: it produces no global interpretability at all, and its post-hoc saliency scores are well below \method on explanation quality (Table~\ref{tab:main}). The additional training cost is therefore precisely what improves interpretability. Meanwhile, comparing \textsc{BD-GRPO} with Vanilla \textsc{B-GRPO} shows a 2--5$\times$ speed-up in total training time (3--5$\times$ when measured to peak, Fig.~\ref{fig:rl_wallclock}), greatly offsetting the overall expense relative to a straightforward application of RL-based fine-tuning.

\begin{table*}[!tbp]
\centering
\caption{Training cost (GPU rows: 4$\times$H100 wall-clock).}
\label{tab:train_cost}
\setlength{\tabcolsep}{4pt}
\small
\begin{tabular}{lrrr}
\toprule
\begin{tabular}[b]{@{}l@{}}Stage\\Method\end{tabular} & ContractNLI & ICLR & MIMIC \\
\quad $n$ / avg.\ tokens & 7191 / 2185 & 8552 / 2242 & 8000 / 3209 \\
\midrule
\multicolumn{4}{l}{\textit{Prompt optimization (API \$)}} \\
\quad GEPA & \$9.9 & \$8.2 & \$10.4 \\
\quad MIPROv2 & \$16.7 & \$12.5 & \$17.7 \\
\quad \spo & \$27.2 & \$13.4 & \$14.1 \\
\multicolumn{4}{l}{\textit{Reasoning Distillation}} \\
\quad teacher sampling (API \$) & \$22.0 & \$25.3 & \$36.0 \\
\quad R-SFT train (GPU h) & 1.51 & 1.25 & 2.47 \\
\multicolumn{4}{l}{\textit{RL}} \\
\quad \textsc{BD-GRPO} train (GPU h) & 1.63 & 2.57 & 3.86 \\
\midrule
\textbf{\method pipeline total} & \textbf{$\sim$\$49 + 3.1 h} & \textbf{$\sim$\$39 + 3.8 h} & \textbf{$\sim$\$50 + 6.3 h} \\
\midrule
\textbf{Cls-Head (GPU h)} & 2.43 & 1.75 & 1.58 \\
\textbf{Vanilla \textsc{B-GRPO} (GPU h)} & 7.61 & 6.42 & 7.70 \\
\bottomrule
\end{tabular}
\end{table*}

\begin{table}[!tb]
\centering
\caption{Agreement between each student and the Stage I classifier on the test sets, as a share of all test inputs (\%). \emph{Fixes}: Stage I is wrong and the student is right; \emph{New errors}: Stage I is right and the student is wrong. Agreement is not meant to reach 100\%, since Stage I itself errs on 17--26\% of inputs. CNLI: ContractNLI; there, the remaining inputs are those where both are wrong with different labels.}
\label{tab:agreement}
\small
\setlength{\tabcolsep}{3.5pt}
\begin{tabular}{llrrr}
\toprule
 & & Agreement & \multicolumn{2}{c}{Different label (\%)} \\
\cmidrule(lr){4-5}
Dataset & Model & (\%) & Fixes ($\uparrow$) & New errors ($\downarrow$) \\
\midrule
CNLI & R-SFT (II) & 89.9 & 5.9 & 3.1 \\
 & \method (III) & 84.6 & 10.2 & 4.4 \\
\midrule
ICLR & R-SFT (II) & 89.0 & 7.2 & 3.8 \\
 & \method (III) & 89.5 & 7.0 & 3.5 \\
\midrule
MIMIC & R-SFT (II) & 86.5 & 6.8 & 6.7 \\
 & \method (III) & 87.1 & 8.7 & 4.2 \\
\bottomrule
\end{tabular}
\end{table}

\paragraph{Inference cost.}
\looseness=-1 Table~\ref{tab:infer_cost} reports inference latency. Each method uses its most favorable 1$\times$H100 configuration, with results averaged over each test set. To make the local GPU cost comparable with the gpt-4.1-mini API, we price H100 time at AWS p5.48xlarge on-demand rates, \$6.88 per GPU-hour.

\looseness=-1 Serving \method costs 5.6--7.1$\times$ less than API prompting and is only 1.4--2.5$\times$ more expensive than a label-only classifier: long-input prefilling dominates the cost, the extra 130--350 reasoning tokens add little, and vLLM's continuous batching keeps the GPU saturated. \method pays for its explanation only once, since the reasoning and the label are produced in the same generation. Cls-Head + Saliency is the opposite: it predicts in a forward pass and then explains in a backward one, and that backward pass must retain the activations of every layer, increasing both memory and compute usage significantly. Once the baseline classifier is also required to explain, \method becomes 1.5--2.7$\times$ cheaper.

\begin{table*}[!tbp]
\centering
\caption{Batch inference throughput and cost (1$\times$H100); \textbf{bold}: best among methods that output an explanation.}
\label{tab:infer_cost}
\setlength{\tabcolsep}{3pt}
\small
\begin{tabular}{l*{3}{rrrr}}
\toprule
 & \multicolumn{4}{c}{ContractNLI} & \multicolumn{4}{c}{ICLR} & \multicolumn{4}{c}{MIMIC} \\
\cmidrule(lr){2-5}\cmidrule(lr){6-9}\cmidrule(lr){10-13}
Method & mF1 & Judge & samp/s & \$/1k & mF1 & Judge & samp/s & \$/1k & mF1 & Judge & samp/s & \$/1k \\
\midrule
\multicolumn{13}{l}{\textit{Without explanation}} \\
Cls-Head, label only (Qwen3-4B) & 0.847 & -- & 17.2 & \$0.11 & 0.825 & -- & 15.6 & \$0.12 & 0.630 & -- & 11.8 & \$0.16 \\
\midrule
\multicolumn{13}{l}{\textit{With explanation}} \\
Zero-shot CoT (gpt-4.1-mini API) & 0.743 & 4.31 & --- & \$1.14 & 0.789 & 3.18 & --- & \$1.69 & 0.517 & 3.58 & --- & \$1.73 \\
Cls-Head + saliency (Qwen3-4B) & 0.847 & 3.43 & 4.4 & \$0.43 & \textbf{0.825} & 2.97 & 4.2 & \$0.45 & 0.630 & 3.36 & 3.2 & \$0.59 \\
\method, reasoning (Qwen3-4B) & \textbf{0.851} & \textbf{4.48} & \textbf{12.3} & \textbf{\$0.16} & \textbf{0.825} & \textbf{3.44} & \textbf{6.4} & \textbf{\$0.30} & \textbf{0.652} & \textbf{4.31} & \textbf{7.4} & \textbf{\$0.26} \\
\bottomrule
\end{tabular}
\end{table*}
\subsection{Agreement between the Rulebook and the Student}
\label{app:agreement}
The student never sees the rulebook at inference, and Stage III is trained to go beyond it. We therefore compare the student's test predictions with those of the Stage I classifier, which applies the learned rulebook (Definition~\ref{def:composition}). Table~\ref{tab:agreement} splits the test inputs by whether the two give the same label, and splits the disagreements by whether the student fixes a Stage I error or turns a correct Stage I label into a wrong one.

The two agree on 85--90\% of test inputs, so the rulebook describes most of the student's decisions. Where they disagree, fixes outnumber new errors, about two to one for \method on all three datasets. On MIMIC, RL shifts this balance: the disagreements of R-SFT split evenly (6.8\% fixes vs.\ 6.7\% new errors), while those of \method lean toward fixes (8.7\% vs.\ 4.2\%). The gap between the two columns equals the student's accuracy gain over Stage I.

\subsection{SPO4SOP Validation Trajectory}\label{app:spoprogress}

Fig.~\ref{fig:spo-val-traj} shows val macro-$F_1$ over SPO4SOP iterations on all three datasets. Because each iteration's subset search ranges over the entire accumulated rule pool and ranks candidates by validation macro-$F_1$, the previously selected subset always remains a feasible candidate; the trajectory is therefore non-decreasing, up to small fluctuations induced by the beam-search approximation.

\begin{figure}[!t]
  \centering
  \includegraphics[width=\linewidth]{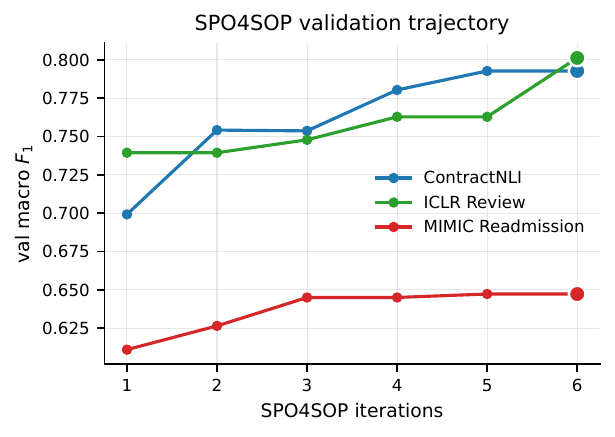}
  \caption{Val macro-$F_1$ across SPO4SOP iterations, per dataset.}
  \label{fig:spo-val-traj}
\end{figure}

\begin{figure*}[!t]
  \centering
  {\setlength{\fboxsep}{2pt}%
  \fbox{\parbox[t]{\dimexpr\linewidth-4pt\relax}{\small\sloppy\emergencystretch=2em
    \textbf{(a) Example rule from the SOP (excerpt):} \hfill \textcolor{purple!70!black}{\textsf{[Global Explanation]}}\\
    \textbf{Trigger Pattern:} Predict \texttt{reject} when reviewer comments contain rejection cues (\hly{limited novelty / missing comparisons}). \ldots\\
    \textbf{Exceptions:} \hlg{no reviewer explicitly recommends rejection} and the criticisms are fixable / exploratory. \ldots\vspace*{1pt}%
  }}\\[1pt]
  \fcolorbox{gray!50}{gray!8}{\parbox[t]{\dimexpr\linewidth-5pt\relax}{\small\sloppy\emergencystretch=2em
    \textbf{(b) ICLR Review (excerpt).}\,\,
    \textit{Reviewer 1:} ``novelty is \emph{limited}; missing comparisons to Chen et al.\ (2018), nn-RNN, unitary-RNNs.'' \,\,\textit{Reviewer 2:} ``interesting exploration of non-normal RNN initialization; plots small; missing reference to Henaff (2016).'' \,\,\textit{Reviewer 3:} ``\emph{method is promising} and paper well written, but comparison with prior work missing.''
  }}\\[1pt]
  \begin{minipage}[c]{0.46\linewidth}
    \fbox{\parbox[t][2.0cm][c]{\dimexpr\linewidth-4pt\relax}{\small\setlength{\fboxsep}{1pt}%
      \textbf{Zero-shot CoT reasoning:}\\
      ``Several missing comparisons (Chen 2018, Henaff 2016, nn-RNN) and \hly{limited novelty} weaken the claim; the paper is not yet ready for acceptance.''\\
      \textit{Predicted:} \texttt{reject}\,\,\ding{55}
    }}%
  \end{minipage}\hfill
  \textbf{\large v.s.}\hfill
  \begin{minipage}[c]{0.46\linewidth}
    \fbox{\parbox[t][2.0cm][c]{\dimexpr\linewidth-4pt\relax}{\small\setlength{\fboxsep}{1pt}%
      \textbf{\method (RL) reasoning:} \hfill \textcolor{purple!70!black}{\textsf{[Local Explanation]}}\\
      ``Critiques cite \hly{limited novelty / missing comparisons}, but \hlg{no reviewer explicitly rejects} and concerns are fixable --- net mixed-to-positive.''\\
      \textit{Predicted:} \texttt{accept} (\checkmark)
    }}%
  \end{minipage}}%
  \caption{\textbf{\method at a glance on ICLR Review.}
  \textbf{(a)} a rule excerpt from our \textit{learned} rulebook (full rule in Apdx.~\ref{app:rules-spo4sop}, Rule~1);
  \textbf{(b)} on a test case: Zero-shot CoT overfires on the surface critique phrases (``limited novelty,'' ``missing comparisons''), whereas \method\ correctly invokes the rule's \emph{exception} (no explicit reject + fixable concerns).}
  \label{fig:case_review}
\end{figure*}

\begin{figure*}[!t]
  \centering
  {\setlength{\fboxsep}{2pt}%
  \fbox{\parbox[t]{\dimexpr\linewidth-4pt\relax}{\small\sloppy\emergencystretch=2em
    \textbf{(a) Example rule from the SOP (excerpt):} \hfill \textcolor{purple!70!black}{\textsf{[Global Explanation]}}\\
    \textbf{Trigger Pattern:} Predict \texttt{contradiction} when the cleanup clause requires \hly{return\ldots or destroy all} Confidential Information and the hypothesis claims the party may retain some. \ldots\\
    \textbf{Exceptions:} Do not apply if the contract \hlg{allows retention of specified categories} (archival / legal / backup). \ldots\vspace*{1pt}%
  }}\\[1pt]
  \fcolorbox{gray!50}{gray!8}{\parbox[t]{\dimexpr\linewidth-5pt\relax}{\small\sloppy\emergencystretch=2em
    \textbf{(b) ContractNLI test case.}\,\,
    \textit{Hypothesis:} ``Receiving Party may \emph{retain some} Confidential Information even after the return or destruction of Confidential Information.''\,\,
    \textit{Contract:} ``Recipient shall return all Confidential Information (including any copies thereof) \emph{or certify the destruction} thereof.''
  }}\\[1pt]
  \begin{minipage}[c]{0.46\linewidth}
    \fbox{\parbox[t][2.0cm][c]{\dimexpr\linewidth-4pt\relax}{\small\setlength{\fboxsep}{1pt}%
      \textbf{Zero-shot CoT reasoning:}\\
      ``Recipient must \hly{return or destroy all CI} on request, but \hlr{the contract is silent on retention} afterwards --- neither entails nor contradicts.''\\
      \textit{Predicted:} \texttt{not\_mentioned}\,\,\ding{55}
    }}%
  \end{minipage}\hfill
  \textbf{\large v.s.}\hfill
  \begin{minipage}[c]{0.46\linewidth}
    \fbox{\parbox[t][2.0cm][c]{\dimexpr\linewidth-4pt\relax}{\small\setlength{\fboxsep}{1pt}%
      \textbf{\method (RL) reasoning:} \hfill \textcolor{purple!70!black}{\textsf{[Local Explanation]}}\\
      ``Cleanup clause requires \hly{return\ldots or certify destruction of all} CI --- an all-encompassing obligation that contradicts `may retain some'.''\\
      \textit{Predicted:} \texttt{contradiction} (\checkmark)
    }}%
  \end{minipage}}%
  \caption{\textbf{\method at a glance on ContractNLI.}
  \textbf{(a)} a rule excerpt from our \textit{learned} rulebook (full rule in Apdx.~\ref{app:rules-spo4sop}, Rule~2);
  \textbf{(b)} on a test case: Zero-shot CoT reads the cleanup clause but wrongly calls the contract ``silent'' on retention, whereas \method\ applies Rule~2 (return/destroy-all contradicts ``may retain some'').}
  \label{fig:case_contractnli}
\end{figure*}

\subsection{Additional Case Studies}
\label{app:cases}

Beyond the MIMIC Readmission teaser (Fig.~\ref{fig:crownjewel}), Figs.~\ref{fig:case_review} and~\ref{fig:case_contractnli} illustrate the same global-rule / local-reasoning interplay on the other two datasets. Each figure pairs an excerpt of a learned SOP rule with a real test case and side-by-side reasoning traces from zero-shot CoT and \method.

\paragraph{ICLR Review (Fig.~\ref{fig:case_review}).}
Panel~(a) shows Rule~1 from the learned SOP (full rule in Table~\ref{tab:iclr-rule1}): predict \texttt{reject} when reviewer comments contain rejection cues such as ``limited novelty'' or ``missing comparisons'', \emph{with the exception} that no reviewer explicitly recommends rejection and the raised concerns are fixable or exploratory. Panel~(b) presents a test case where three reviewers raise concerns about ``limited novelty'' and missing comparisons but none recommends outright rejection, and the criticisms remain fixable. In (b) right, \method correctly invokes the rule's exception---identifying the absence of an explicit reject signal and the fixable nature of the concerns---and predicts \texttt{accept}. In contrast, (b) left shows zero-shot CoT latching onto the surface critique phrases (``limited novelty,'' ``missing comparisons'') and mispredicting \texttt{reject}, without weighing that the aggregate reviewer signal is mixed-to-positive.

\paragraph{ContractNLI (Fig.~\ref{fig:case_contractnli}).}
Panel~(a) shows Rule~2 from the learned SOP (full rule in Table~\ref{tab:contractnli-rule2}): predict \texttt{contradiction} when a contract's cleanup clause requires the receiving party to return or destroy \emph{all} Confidential Information but the hypothesis claims the party may retain some of it, \emph{with the exception} that the contract expressly allows retention of specified categories (archival, legal, backup, etc.). Panel~(b) presents a test case where the contract requires ``recipient shall return all Confidential Information (including any copies thereof) or certify the destruction thereof''---an all-encompassing return/destroy obligation that contradicts the hypothesis ``Receiving Party may retain some Confidential Information.'' In (b) right, \method recognises the rule's trigger pattern (return\ldots or certify destruction of \emph{all} CI) and correctly outputs \texttt{contradiction}. In contrast, (b) left shows zero-shot CoT noting the strict return/destroy language but then misjudging the contract as ``silent on retention,'' incorrectly outputting \texttt{not\_mentioned}.

\looseness=-1 In both cases the failure mode mirrors the MIMIC example in \S\ref{ssec:study}: zero-shot CoT over-fires on a surface trigger or under-fires by missing an applicable rule, while \method invokes the relevant rule together with its exception. As in MIMIC, \method produces its reasoning \emph{without} the SOP in-context at inference time, indicating that the rules have been internalised through distillation and preserved through RL.

\section{\spo Learned Rules}
\label{app:rules-spo4sop}
The complete learned rulebooks and the optimized instruction prompts of the MIPROv2 and GEPA baselines are released with our code. The baseline prompts were obtained on the same 180-row class-balanced train subset (Apdx.~\ref{app:data}); unlike \spo's modular rule sets, each is a single monolithic instruction per dataset.
A rule assigns the label in its title only when it fires; if no rule fires, the classifier falls back to a default label, which is \emph{not mentioned} for ContractNLI, \emph{accept} for ICLR Review, and \emph{not readmitted} for MIMIC. Inside the rule texts, labels appear as class indices: 0/1/2 stand for not mentioned/entailment/contradiction in ContractNLI, and 0/1 stand for reject/accept in ICLR Review and for not readmitted/readmitted in MIMIC. Many examples carry the opposite label. These are cases on which the rule should stay silent, and \spo collects them as exceptions whenever the rule fires wrongly during training (false coverage in Algo.~\ref{algo:spo4sop}).

\subsection{ContractNLI (5 rules)}
Tables~\ref{tab:contractnli-rule1} and~\ref{tab:contractnli-rule2} show two of the five rules learned by \spo for ContractNLI.

\begin{table*}[!tb]
\centering
\begin{prompttemplatebox}{ContractNLI Rule 1 (fires $\to$ entailment): Paraphrased NDA Duty or Definition Still Entails}
\footnotesize
\textbf{Trigger Pattern:} Apply label 1 when the hypothesis is a semantic paraphrase of an express contractual duty or definition, even if the wording is different. Look for 2--3 indicators such as:\\
-- the contract uses different but equivalent legal phrasing, e.g.\ ``oral or written form'' = verbally conveyed;\\
-- party-role equivalents like ``Recipient'' / ``Receiving Party'';\\
-- broader wording that clearly covers the same fact, such as ``existence and nature of this Agreement'' covering disclosure that the agreement was agreed or negotiated.

\textbf{Exceptions:} Do not apply this rule if:\\
-- the contract merely discusses a related topic but does not clearly cover the hypothesis;\\
-- the hypothesis requires a stronger or different obligation than the source text;\\
-- the contract language is only tangentially related and not a true semantic match.

\textbf{Example.}\\
\textit{Source text:} ``Recipient agrees to keep the existence and nature of this Agreement confidential.''\\
\textit{Wrong:} Label 0 (NotMentioned) for ``The Receiving Party shall not disclose the fact that the Agreement was agreed or negotiated.''\\
\textit{Correct:} Label 1 (Entailment) because confidentiality of the agreement's existence and nature supports nondisclosure of the fact it was negotiated.
\end{prompttemplatebox}
\caption{ContractNLI learned rule 1 (SPO4SOP output).}
\label{tab:contractnli-rule1}
\end{table*}
\begin{table*}[!tb]
\centering
\begin{prompttemplatebox}{ContractNLI Rule 2 (fires $\to$ contradiction): Return/Destroy Clauses Do Not Imply a Right to Make Copies}
\footnotesize
\textbf{Trigger Pattern:} Apply label 2 when the hypothesis claims the Receiving Party may retain, keep, possess, or make/use copies of Confidential Information after return, destruction, or termination, but the contract requires one or more of the following:\\
-- return all confidential or proprietary materials upon request or termination;\\
-- destroy all such materials or copies;\\
-- not keep copies, duplicates, or retained versions of the information;\\
-- surrender or return all documents/copies without any retention right.\\
Use this rule when the clause shows a full return/destruction framework and the hypothesis tries to derive a right to retain or possess copies from language that only addresses returning/destroying existing copies.

\textbf{Exceptions:} Do not apply this rule if:\\
1. The contract expressly allows retention of specific archival, legal, regulatory, or backup copies.\\
2. The contract is only about return/destruction of existing materials and is silent on post-termination retention; if the hypothesis is about retaining copies, that may be NotMentioned unless retention is prohibited.\\
3. The hypothesis is about the right to create copies in some circumstances, rather than to keep existing copies after termination; that is a different topic.\\
4. The contract never addresses copies at all, and the only relevant issue is general copying permission.\\
5. The hypothesis concerns a different subject unrelated to post-termination possession, return, destruction, or copy retention.

\textbf{Examples.}\\
\textit{Source text:} ``Upon request or termination, Receiving Party shall return all Information and destroy all materials containing Information and shall not keep copies or duplicates thereof.''\\
\textit{Wrong:} Label 0 for ``The Receiving Party may retain some Confidential Information even after return or destruction.''\\
\textit{Correct:} Label 2, because the contract expressly requires return/destruction and forbids keeping copies or duplicates, which conflicts with any post-return retention right.

\textit{Source text:} ``Upon termination, Recipient shall promptly return or destroy all Confidential Information and all copies thereof.''\\
\textit{Wrong:} Label 2 for ``Recipient is permitted to make copies of Confidential Information for use during the term.''\\
\textit{Correct:} Label 0, because the clause only requires returning or destroying existing copies after termination; it does not address whether copying is allowed in the first place.
\end{prompttemplatebox}
\caption{ContractNLI learned rule 2 (SPO4SOP output).}
\label{tab:contractnli-rule2}
\end{table*}
\subsection{ICLR Review (2 rules)}
Tables~\ref{tab:iclr-rule1} and~\ref{tab:iclr-rule2} show the two rules learned by \spo for ICLR Review. Note that the name of Rule 1, written by the Update LLM, describes its exceptions rather than the cases it rejects.

\begin{table*}[!tb]
\centering
\begin{prompttemplatebox}{ICLR Review Rule 1 (fires $\to$ reject): Positive-leaning Mixed Reviews Override Critique}
\footnotesize
\textbf{Trigger Pattern:} Apply this rule when reviewer comments contain both praise and criticism, but the model is tempted to predict rejection just because it sees negative phrases. Typical rejection cues include:\\
-- explicit reject / lean reject / cannot recommend acceptance / not strong enough / not the right venue,\\
-- criticism of novelty, contribution size, positioning, applicability, baselines, clarity, or rigor,\\
-- post-rebuttal remarks that remain skeptical.

\textbf{Exceptions:} Do not apply this rule if any of the following hold:\\
1. One or more reviewers explicitly recommend acceptance, marginal-accept, or clearly revise upward after rebuttal.\\
2. The overall discussion is mixed-to-positive: substantive praise for novelty, usefulness, clarity, significance, strong results, or impact outweighs the concerns.\\
3. The criticisms are mainly fixable or exploratory issues (extra ablations, more comparisons, clarification, code, presentation, sensitivity analysis) rather than fatal methodological flaws.\\
4. Positive language is substantive endorsement of the paper's contribution, not mere politeness.\\
5. A single rejection-leaning review should not dominate if other reviewers are positive or acceptance-leaning, especially when the negative review itself is mixed and ends with openness to revision.

\textbf{Examples.}\\
\textit{Source text:} ``The idea is interesting and the paper is well written, but the contribution feels incremental, the experiments are insufficient, and I still cannot recommend acceptance.''\\
\textit{Wrong:} Predict label 1 because of ``interesting'' and ``well written.''\\
\textit{Correct:} Predict label 0 because the final stance and substantive critiques are rejection-leaning.

\textit{Source text:} ``Overall, I think this is a good paper and I recommend acceptance. There are some concerns about ablations and presentation, but they seem fixable.''\\
\textit{Wrong:} Predict label 0 because of the concerns and criticism.\\
\textit{Correct:} Predict label 1 because the review explicitly recommends acceptance and the negatives are minor.

\textit{Source text:} ``The method is promising and the results are strong. I have concerns about missing comparisons and would like more analysis, but I revised my assessment from marginal-reject to marginal-accept.''\\
\textit{Wrong:} Predict label 0 because of the earlier skepticism.\\
\textit{Correct:} Predict label 1 because the final recommendation is acceptance-leaning and the concerns are not decisive.

\textit{Source text:} ``I enjoyed reading this paper. The idea is natural, useful, and has interesting potential applications. I would be willing to upgrade my review if points 2 and 3 are addressed.''\\
\textit{Wrong:} Predict label 0 because of the concerns about methodology and missing ablations.\\
\textit{Correct:} Predict label 1 because the review is substantively positive and explicitly open to upgrading.

\textit{Source text:} ``This is a valuable and novel benchmark; the results are promising and competitively strong. Although one reviewer is more critical, the overall set of reviews is positive and the remaining concerns are mostly about clarification and extra comparisons.''\\
\textit{Wrong:} Predict label 0 because of the critical review and requests for more analysis.\\
\textit{Correct:} Predict label 1 because the aggregate signal is acceptance-leaning and the issues are not fatal.
\end{prompttemplatebox}
\caption{ICLR Review learned rule 1 (SPO4SOP output).}
\label{tab:iclr-rule1}
\end{table*}
\begin{table*}[!tb]
\centering
\begin{prompttemplatebox}{ICLR Review Rule 2 (fires $\to$ reject): Reject Only When Negative Signal Dominates, Not When Positives Outweigh Cautions}
\footnotesize
\textbf{Trigger Pattern:} Apply this rule when the review set has some positive or acceptance-leaning language, but the overall consensus is still clearly rejection-leaning. Typical cues:\\
-- one reviewer is favorable while others explicitly reject or remain strongly unconvinced,\\
-- multiple reviews say the paper is only a minor extension, lacks novelty, or has weak justification/experiments,\\
-- criticisms focus on core merit: novelty, validity, robustness, theoretical support, or impact,\\
-- the aggregate tone is mixed-to-negative overall, not merely cautious.

\textbf{Exceptions:} Do not apply this rule if:\\
1. Two or more reviewers explicitly recommend acceptance, or clearly lean acceptance, and the negative comments are secondary.\\
2. The unfavorable review is an outlier while the other reviews are broadly positive or supportive.\\
3. The main issues are requests for more ablations, baselines, sensitivity tests, clearer framing, or related-work discussion, without questioning the paper's central merit.\\
4. Reviews contain explicit praise such as ``novel,'' ``promising,'' ``interesting,'' ``useful,'' ``state of the art,'' or ``good contribution,'' and these positives are echoed by multiple reviewers.\\
5. Mixed reviews still conclude with practical value or endorsement of the approach, especially when comments emphasize presentation/coverage limitations rather than fatal flaws.

\textbf{Examples.}\\
\textit{Source text:} ``Reviewer 1: good writing, but the gains are not convincing. Reviewer 2: this is a minor extension with weak experiments. Reviewer 3: promising, but I still recommend rejection due to low novelty.''\\
\textit{Wrong:} Predict label 1 because one reviewer used positive language.\\
\textit{Correct:} Predict label 0 because the strongest and majority signal is rejection-leaning.

\textit{Source text:} ``Reviewer 1: the idea is novel and useful, but I need more ablations. Reviewer 2: promising results and a nice theoretical foundation; lean reject only because of missing comparisons. Reviewer 3: innovative approach with good potential. Reviewer 4: nice contribution, state-of-the-art performance, and a variety of datasets.''\\
\textit{Wrong:} Predict label 0 because some reviewers mention concerns and one leans reject.\\
\textit{Correct:} Predict label 1 because multiple reviewers are clearly positive and the concerns are secondary.
\end{prompttemplatebox}
\caption{ICLR Review learned rule 2 (SPO4SOP output).}
\label{tab:iclr-rule2}
\end{table*}
\subsection{MIMIC Readmission (4 rules)}
Tables~\ref{rule:mimic_diagnostic_workup} and~\ref{tab:mimic-rule2} show two of the four rules learned by \spo for MIMIC Readmission.

\begin{table*}[!tb]
\centering
\begin{prompttemplatebox}{MIMIC Readmission Rule 1 (fires $\to$ readmit): Active Malignancy/Transplant Only When Disease Is Unstable or Intensively Managed}
\footnotesize
\textbf{Trigger Pattern:} Apply label 1 when discharge occurs in the setting of ongoing high-acuity oncology or transplant disease, especially if any of the following are present:\\
-- progressive/metastatic cancer or recurrent leukemia/lymphoma,\\
-- recent chemotherapy, immunosuppression, or transplant-related complications,\\
-- unresolved symptoms or treatment changes requiring very close follow-up,\\
-- discharge after a fragile course with new medications, planned labs, or next-day specialty review.

\textbf{Exceptions:} Do not apply label 1 for:\\
1. Cancer history or transplant history alone when the disease is explicitly stable, treatment is complete or maintenance-only, and there is no acute cancer-related complication.\\
2. Stable or improving non-oncologic admissions where cancer is only a background comorbidity and the main problem is something else (e.g., GI bleed, anemia, falls, orthostasis, cellulitis).\\
3. Remote transplant with mild AKI/chronic allograft dysfunction that is managed conservatively and does not require additional inpatient transplant workup.\\
4. Routine outpatient monitoring (labs, tacrolimus levels, PCP/nephrology follow-up) or standard short oral antibiotics when the patient is clinically improving at discharge.\\
5. Planned/maintenance therapy without treatment toxicity or progression, including patients tolerating scheduled chemotherapy or chronic hormonal therapy, with resolved transient issues and no urgent specialty follow-up need.

\textbf{Examples.}\\
\textit{Source text:} ``Metastatic castration-resistant prostate cancer with spinal/cauda equina compression, AKI, Foley catheter, aspiration risk, and discharge to home despite rehab recommendation.''\\
\textit{Wrong:} Predict label 0 because symptoms improved and discharge was documented as stable.\\
\textit{Correct:} Predict label 1 because active progressive malignancy plus neurologic injury and device dependence create high 30-day readmission risk.

\textit{Source text:} ``History of ESRD s/p remote DDKT with mild AKI/hyperkalemia (Cr 2.2 $\to$ 2.1), transplant team felt no additional inpatient workup was needed, tacrolimus adjusted, outpatient CBC/Chem10/tacro level planned, discharge on oral antibiotics after cellulitis improved.''\\
\textit{Wrong:} Predict label 1 because transplant history, immunosuppression, and planned labs imply high-acuity transplant instability.\\
\textit{Correct:} Predict label 0 because this is remote transplant history with mild, improving renal issues and routine outpatient follow-up.

\textit{Source text:} ``Seemingly stable prostate cancer on maintenance abiraterone/prednisone/goserelin; admitted for GI bleed, anemia, falls, orthostasis, and demand NSTEMI; oncology follow-up not urgent.''\\
\textit{Wrong:} Predict label 1 because any active cancer or hormonal therapy triggers the rule.\\
\textit{Correct:} Predict label 0 because the malignancy is stable, treatment is maintenance-only, and the admission was for non-oncologic problems.

\textit{Source text:} ``Refractory CLL on scheduled inpatient chemotherapy with fever and cough that resolved, CT/CXR negative, ambulatory and independent at discharge, routine supportive meds and follow-up arranged.''\\
\textit{Wrong:} Predict label 1 because leukemia and chemo automatically imply imminent readmission.\\
\textit{Correct:} Predict label 0 because this was a planned treatment admission with clinical stability and no unresolved complication at discharge.
\end{prompttemplatebox}
\caption{MIMIC Readmission learned rule 1 (SPO4SOP output).}
\label{rule:mimic_diagnostic_workup}
\end{table*}
\begin{table*}[!tb]
\centering
\begin{prompttemplatebox}{MIMIC Readmission Rule 2 (fires $\to$ readmit): Complicated Postoperative Discharge With Foley or Unresolved Recovery Milestone}
\footnotesize
\textbf{Trigger Pattern:} Apply label 1 when the discharge summary shows a recent major surgery plus unresolved short-term postoperative complications that still require active management after discharge. Strong indicators include:\\
-- failed voiding trial, urinary retention, high post-void residuals, or discharge with an indwelling Foley catheter,\\
-- need for catheter-related prophylaxis, device care, or planned early catheter removal follow-up,\\
-- persistent postoperative issues such as wound drainage, infection concern, ileus, uncontrolled pain, or significant anemia requiring close follow-up.

\textbf{Exceptions:} Do not apply this rule for routine postoperative discharges with no unresolved complication, no device dependence, and no early follow-up for a specific recovery problem.

\textbf{Example.}\\
\textit{Source text:} ``Total abdominal hysterectomy with sling/cystoscopy complicated by urinary retention; failed voiding trial, discharged home with Foley catheter, nitrofurantoin prophylaxis, and follow-up for catheter removal.''\\
\textit{Wrong:} Predict label 0 because the patient was discharged home and described as stable.\\
\textit{Correct:} Predict label 1 because the postoperative course was not resolved and Foley/device dependence with early follow-up creates high 30-day return risk.
\end{prompttemplatebox}
\caption{MIMIC Readmission learned rule 2 (SPO4SOP output).}
\label{tab:mimic-rule2}
\end{table*}

\end{document}